\documentclass{article}

\usepackage{microtype}
\usepackage{graphicx}
\usepackage{subfigure}
\usepackage{booktabs} 
\usepackage{tablefootnote}
\usepackage{hyperref}

\usepackage[accepted]{icml2024}

\usepackage{amsmath}
\usepackage{amssymb}
\usepackage{mathtools}
\usepackage{amsthm}
\usepackage{multirow}
\usepackage[capitalize,noabbrev]{cleveref}

\theoremstyle{plain}
\newtheorem{theorem}{Theorem}[section]
\newtheorem{proposition}[theorem]{Proposition}
\newtheorem{lemma}[theorem]{Lemma}

\newtheorem{example}[theorem]{Example}

\newtheorem{condition}[theorem]{Condition}
\theoremstyle{condition}

\newtheorem{definition}[theorem]{Definition}
\theoremstyle{definition}

\theoremstyle{remark}
\DeclareMathOperator*{\argmin}{arg\,min}

\newcommand{\eins}{\boldsymbol{1}}

\usepackage[textsize=tiny]{todonotes}

\icmltitlerunning{Better Locally Private Sparse Estimation Given Multiple Samples Per User}

\begin{document}
	
	\twocolumn[
	\icmltitle{Better Locally Private Sparse Estimation Given Multiple Samples Per User}
	
	
	
	
	\begin{icmlauthorlist}
		\icmlauthor{Yuheng Ma}{rucstat}
		\icmlauthor{Ke Jia}{rucstat}
		\icmlauthor{Hanfang Yang}{rucstatcenter,rucstat}
	\end{icmlauthorlist}
	
	\icmlaffiliation{rucstat}{School of Statistics, Renmin University of China}
	\icmlaffiliation{rucstatcenter}{Center for Applied Statistics, Renmin University of China}
	
	\icmlcorrespondingauthor{Hanfang Yang}{hyang@ruc.edu.cn}
	
	\icmlkeywords{User Level Local Differential Privacy, Sparse Linear Regression}
	
	\vskip 0.3in
	]

	\printAffiliationsAndNotice{\icmlEqualContribution} 

	\begin{abstract}

		Previous studies yielded discouraging results for item-level locally differentially private linear regression with $s^*$-sparsity assumption, where the minimax rate for $nm$ samples is $\mathcal{O}(s^{*}d / nm\varepsilon^2)$. 
		This can be challenging for high-dimensional data, where the dimension $d$ is extremely large.
		In this work, we investigate user-level locally differentially private sparse linear regression.
		We show that with $n$ users each contributing $m$ samples, the linear dependency of dimension $d$ can be eliminated, yielding an error upper bound of $\mathcal{O}(s^{*2} / nm\varepsilon^2)$.
		We propose a framework that first selects candidate variables and then conducts estimation in the narrowed low-dimensional space, which is extendable to general sparse estimation problems with tight error bounds.
		Experiments on both synthetic and real datasets demonstrate the superiority of the proposed methods. 
		Both the theoretical and empirical results suggest that, with the same number of samples, locally private sparse estimation is better conducted when multiple samples per user are available.

	\end{abstract}

	\section{Introduction}
	
	Local differential privacy (LDP) \citep{kairouz2014extremal, duchi2018minimax}, a variant of differential privacy (DP) \citep{dwork2006calibrating}, has gained considerable attention in recent years. 
	LDP assumes that each sample is possessed by a data holder, who privatizes their data before it is collected by the curator. 
	Offering a stronger sense of privacy protection compared to central DP, learning under LDP often encounters challenges such as slow convergence, high demand for local machine capacity, and limited accessibility to basic techniques \citep{duchi2018minimax, tramer2022considerations, ma2024decision}, which obstruct the theoretical analysis and practical implementation of LDP learning.
	
	Fortunately, in some scenarios, each user may possess multiple samples, which can serve as a way to overcome these difficulties. 
	This is known as user-level LDP (ULDP) \citep{acharya2023discrete, bassily2023user}. 
	Research has demonstrated performance improvement in intentionally designed models when each user has multiple samples, from
	both the central DP perspective \cite{liu2020learning, ghazi2021user, levy2021learning, narayanan2022tight, ghazi2023user} and 
	the LDP perspective \cite{girgis2022distributed, acharya2023discrete, bassily2023user}.
	In most cases (for ULDP), the improvement lies in the effective sample size: if there are $n$ users with $m$ samples and privacy budget $\varepsilon$, the problem is as tractable as having $nm$ users with one sample and privacy budget $\varepsilon$. 
	See Table \ref{tab:summarizetheoreticalresults} for a summary.

	We proceed to ask the following question: 
	\textit{Besides effective sample size, does having multiple samples per user offer benefits?}
	If the answer to this question is affirmative, it holds practical significance.
	For instance, when designing data collection schemes, the primary focus should be on users capable and willing to provide multiple samples.
	Moreover, if a significant number of users lack trust in the data collector but are willing to share information within small groups (such as family or company), then better mechanisms can be devised for conducting the learning process.

	In this work, we offer an affirmative response to the question from the perspective of sparse estimation.
	Sparse estimation stands as a crucial task in modern machine learning, especially when dealing with high-dimensional data where structured assumptions like sparsity can significantly enhance performance.
	Particularly, we study sparse linear regression.
	We first elucidate why the minimax lower bound fails to hold when each user possesses multiple samples and provide a lower bound for ULDP (Theorem \ref{thm:ourlowerbound}). 
	Subsequently, we introduce an algorithm structured as follows: half of the users perform local variable selection and aggregate their findings to identify the support of non-zero variables.
	Under mild assumptions, we establish theoretical guarantees for both local selection (Proposition \ref{prop:existenceofgoodselectors}) and aggregation (Proposition \ref{prop:heavyhitterselection}). 
	Then, to conduct estimation on the narrowed space, we propose a sub-optimal multi-round protocol (Theorem \ref{thm:preciseestimation2}) and a two-round protocol (Theorem \ref{thm:preciseestimation}). 
	The latter achieves an estimation error $\mathcal{O}(s^{*2} / nm\varepsilon^2)$.
	Compared to minimax error rate $\mathcal{O}(d s^* / nm\varepsilon^2)$ under LDP, our rate improves by a factor of $s^* / d$ which can be significant for high dimensional data.
	Furthermore, we demonstrate how the latter protocol straightforwardly extendeds to other sparse estimation problems (Theorem \ref{thm:preciseestimationgeneral}).
	
	We summarize our contributions as follows.

	\begin{itemize}
		\item We formalize, for the first time, the advantage of ULDP over LDP by considering the sparse assumption. 
		Our findings reveal that the rates of sparse problems, such as sparse linear regression and sparse mean estimation, do not scale linearly in $d$ under ULDP, which contrasts with previous negative results for LDP.
		\item We provide a general framework for ULDP sparse estimation. Moreover, focusing on linear regression, we devise tailored methods that achieve tight upper bounds. 
		The precise estimation procedures serve as solutions to low-dimensional ULDP linear regression, which are of independent interest.
		\item We conduct experiments on both synthetic and real datasets, with convincing results demonstrating the superiority of our methods.
	\end{itemize}
	
	The article is structured as follows: 
	In Section \ref{sec:ULDPSLR}, we discuss related literature, preliminary knowledge, and minimax results of ULDP sparse linear regression. 
	In Section \ref{sec:mainresults}, we present our solutions.
	In Section \ref{sec:experiments}, we provide experiment results. 
	All technical proofs, detailed algorithms, and additional experiment results are included in the appendix.
	
	\begin{table}[!t]
		\caption{Comparison of error rate between non-private, ULDP, and LDP results. Results assume the true parameter lies within $\ell_{\infty}$ unit ball. 
			Here, we consider sparse regression with beta-min condition, which improves a $\log d$ over the usual case. 
		}
		\label{tab:summarizetheoreticalresults}
		\resizebox{1\linewidth}{!}{
			\renewcommand{\arraystretch}{1}
			\setlength{\tabcolsep}{3pt}
			\begin{tabular}{cccc}
				\toprule
				& \begin{tabular}[c]{@{}c@{}}Non-private\\ ($nm$ samples)\end{tabular} & \begin{tabular}[c]{@{}c@{}}$\varepsilon$-ULDP\\ ($n$ users $m$ samples)\end{tabular} & \begin{tabular}[c]{@{}c@{}}$\varepsilon$-LDP\\ ($nm$ samples)\end{tabular} \\ \midrule
				\begin{tabular}[c]{@{}c@{}}discrete\\ distribution\tablefootnote{\citet{kairouz2016discrete, acharya2023discrete}}   \end{tabular}                & $\sqrt{\frac{k}{nm}}$                                                                    & $\sqrt{\frac{k^2}{nm\varepsilon^2}}$                                                           &  $\sqrt{\frac{k^2}{nm\varepsilon^2}}$                                                           \\
				\begin{tabular}[c]{@{}c@{}}mean\\ estimation\tablefootnote{\citet{duchi2018minimax, bassily2023user}} \end{tabular}                      & $\frac{d}{nm}$                                                                   & $\frac{d^2}{nm\varepsilon^2}$                                                                      & $\frac{d^2}{nm\varepsilon^2}$                                                         \\
				\begin{tabular}[c]{@{}c@{}}\textbf{sparse}\\  \textbf{regression}\tablefootnote{\citet{ndaoud2019interplay, zhu2023improved}}\end{tabular} & $\frac{s^* }{nm}$                                                                    & $\mathbf{\frac{ s^{*2}}{nm\varepsilon^2}}$     \textbf{(ours)}                                                        & $\frac{ d s^{*2} }{nm\varepsilon^2}$                                                            \\ \bottomrule
			\end{tabular}
		}
		\vskip -0.05 in
	\end{table}

	\section{ULDP Sparse Linear Regression}\label{sec:ULDPSLR}

	\subsection{Preliminaries}\label{sec:preliminary}
	We introduce necessary notations. 
	For any vector $x$, let $x^i$ denote the $i$-th element of $x$. 
	Let $x^{\{i_1, \cdots, i_j \}}$ be a slicing vector of $x$, whose $j$-th elements is $x^{i_j}$. 
	Let $\|x\|_p$ be the $\ell_p$ norm of $x$ for $0 \leq p \leq \infty$.
	We will evaluate the estimation error by the squared loss, i.e. $\|\widehat{\beta} - \beta^*\|_2^2$. 
	For matrix $A$, let $\lambda_i(A)$ denote the $i$-th largest singular value of $A$. 
	Throughout this paper, we use the notation $a_n \lesssim b_n$ and $a_n \gtrsim b_n$ to denote that there exist positive constant $c$ and $c'$ such that $a_n \leq c b_n$ and $a_n \geq c' b_n$, for all $n \in \mathbb{N}$.
	We use $a = \mathcal{O}(b)$ if $a \lesssim b$. 
	We denote $a_n\asymp b_n$ if $a_n\lesssim b_n$ and $b_n\lesssim a_n$.
	Let $a\vee b = \max (a,b)$ and $a\wedge b = \min (a,b)$. 
	Besides, for any set $A\subset \mathbb{R}^d$, the diameter of $A$ is defined by $\mathrm{diam}(A):=\sup_{x,x'\in A}\|x-x'\|_2$. 

	Suppose we have $n$ users.
	The $i$-th user has $m$ i.i.d. samples $(X_i, y_i) = \{(X_{i,j}, y_{i,j}),j=1,\cdots,m\}$ from distribution $\mathrm{P}$ on domain $\mathcal{X}\times \mathcal{Y} \subseteq \mathbb{R}^d\times \mathbb{R}$. 
	We consider the classical sparse linear regression.
	Let each $X_{i,j}$ be i.i.d. sub-Gaussian. 
	Moreover, $\Sigma = \mathbb{E}[XX^{\top}]$ denote the covariance matrix of the marginal distribution.
	Assume $C_X^{-1} \leq \lambda_d(\Sigma)\leq \lambda_1(\Sigma)\leq C_X$ for some constant $C_X > 1$. 
	For mean zero sub-Gaussian random variable $\sigma$,  conditional distribution $\mathrm{P}_{Y|X}$ and its coefficients $\beta^*$ are described by 
	\begin{align}\label{equ:modelassumption}
		y = X\beta^* + \sigma, \quad 	\beta^*\in \Omega_{s^*, a}^d = \biggl\{ \|\beta^*\|_0 \leq s^* , & \\  \|\beta^*\|_{\infty} \leq 1, 
		\max_{\beta^{*j} > 0}|\beta^{*j}| & \geq a \biggr\}, \nonumber 
	\end{align}
	which has $s^*$-sparsity and non-zero entries bounded away from 0. 
	Without loss of generality, we assume the first $s^*$ elements of $\beta^*$ are non-zero.

	
	We adopt the following setting for privacy constraints.
	Any estimation of $\beta^*$ is considered as a random variable, while its construction process with respect to the data is \textit{user-level locally differentially private} (ULDP).
	We consider the sequential interactive case where the private observation $U_i $ is decided only by its local samples $(X_i, y_i)$ and previous observations $U_1,\cdots, U_{i-1}$. 
	The rigorous definition of (pure) ULDP is as follows.

	\begin{definition}[\textbf{User-level local differential privacy}]\label{def:uldp}
		Given data $\{ (X_i, y_i) \}_{i = 1}^n$, each $(X_i, y_i)$ is mapped to privatized information $U_i$ which is a random variable on $\mathcal{U}$. 
		Let $\sigma(\mathcal{U})$ be the $\sigma$-field on $\mathcal{U}$. 
		$U_i$ is drawn conditional on $(X_i, y_i)$ via the distribution $\mathrm{R}$$\left(U_i \mid X_i=x, Y_i = y, U_{1:(i-1)} = u_{1:(i - 1)}\right)$. 
		Then the mechanism $\mathrm{R}$ provides \textit{$\varepsilon$-user-level local differential privacy} ($\varepsilon$-ULDP) if
		\begin{align*}
			\frac{\mathrm{R}\left(U_i \in U \mid X_i = x, Y_i = y, U_{1:(i-1)} = u_{1:(i - 1)}\right)}{\mathrm{R}\left(U_i \in U \mid X_i = x^{\prime}, Y_i = y^{\prime} , U_{1:(i-1)} = u_{1:(i - 1)}\right)} \leq  e^{\varepsilon} 
		\end{align*}
		for all $1 \leq i \leq n$,  $U\in \sigma(\mathcal{U})$, $x, x^{\prime} \in \mathcal{X}^m$, and $y, y^{\prime} \in \mathcal{Y}^m$. 
	\end{definition}

	ULDP reduces to the conventional item-level LDP for $m = 1$.
	Besides being more practically reasonable \citep{cummings2022mean}, ULDP is also a more stringent definition than item-level LDP. 
	To achieve $\varepsilon$-ULDP by trivially using group privacy, each item must use a significantly smaller budget $\varepsilon / m$.
	Conversely, on the curator side, inference of any single item is no easier than inference of the whole user, which means each item is as safe as $\varepsilon$-LDP against the curator. 
	As a compromise, each item should expose information to its group mates. 
	The requirement is typically acceptable, such as when each user has multiple records on a personal cellphone, or when data sources can be clustered into small groups where secrete information is safely shared.

	\subsection{Related Work}

	Extensive studies have been conducted focusing on the central DP setting for linear regression model in  low dimensions \citep{wang2018revisiting, avella2021differentially, alabi2022differentially, arora2022differentially, amin2023easy} and high dimensions \citep{kifer2012private, talwar2015nearly, kumar2019differentially, zhang2021high, cai2021cost, hu2022high, khanna2023challenge, khanna2023sparse, raff2023scaling}.
	Despite variations in settings and assumptions, state-of-the-art results \citep{liu2022differential, varshney2022nearly, cai2023private} indicate a general error rate of $\mathcal{O}(s^*\log d / (n\varepsilon^2))$ for squared loss, where the dependency on the dimension of feature space is $\log d$. 
	Thus, $d$ can be exponentially large in $n \varepsilon^2$ to ensure consistent estimation.

	This is not the case in local setting, of which there is still a lack of understanding compared to the central one.
	Several works addressed the problem focusing on the optimization error \citep{smith2017interaction, zheng2017collect}.
	Both works assumed $\mathrm{diam}(\mathcal{X}) \leq 1$ and do not generalize to many practical settings, such as when all features are i.i.d. and therefore $\mathrm{diam}(\mathcal{X}) = \mathcal{O}(\sqrt{d})$.  
	As for statistical estimation, \citet{duchi2018minimax} showed the matching upper and lower bounds for low dimensional, non-interactive linear regression are $\mathcal{O}(d / (n \varepsilon^2))$. 
	\citet{wang2019sparse} first provided the lower bound $\mathcal{O}(d / (n \varepsilon^2))$ for LDP linear regression with 1-sparsity, which is then generalized to $s$-sparsity by \citet{zhu2023improved}.
	In summary, these prohibitive results indicate that there exists no meaningful approach when $d \asymp n\varepsilon^2$, which is often the case in practice.
	
	Our approach utilize a selection-estimation strategy, which is shown to be advantages under many situation. 
	Under non-private setting, \citet{wang2011random, liang2023vsolassobag} select candidate variables by aggregating Lasso fitted on random subsamples, which is also adopted with privacy \citep{kifer2012private}.
	More recently, the strategy has been used for communication-constrained learning \cite{duchi2019lower, barik2020exact, acharya2021estimating}. 
	Note that our method is also communication efficient, as each user sends only $1$ bit of information. 
	\citet{acharya2021estimating} tackled LDP sparse discrete distribution estimation by selecting the support variables.
	However, this is only feasible for such specific problems where users can provide useful information about which variables are potentially non-zero given only one sample.
	Their result does not generalize to other problems.

	During the preparation of the camera-ready version of this paper, \citet{kent2024rate} appeared online and analyzed sparse mean estimation under ULDP.
	We share some results with their conclusions, including the negative results for $m\leq s^*\log d$, the established rates, and a support estimation type estimator. 
	However, we primarily consider the case where $s^*, \log d\lesssim n\varepsilon^2, m \lesssim d$, whereas their analysis is more comprehensive, considering other regions and, more importantly, identifying the phase transition.

	\subsection{Minimax Lower Bound}
	
	We introduce the related minimax results of locally private sparse linear regression. 
	For any loss function $\ell$ (squared loss in our case), the minimax convergence rate is 
	\begin{align*}
		\inf_{\beta} \sup_{\mathrm{P} \in \mathcal{H}} \mathbb{E}_{\mathrm{P}}\left[\ell (\beta^*, \beta(X, y))\right]
	\end{align*}
	where $\mathcal{H}$ is the hypothesis distribution class and $\beta$ is any estimator of $\beta^*$. 
	The minimax lower bound for sparse linear regression under LDP is well explored in \citet{wang2019sparse} and \citet{zhu2023improved}.

	\begin{proposition}[\textbf{LDP lower bound}]\label{thm:lowerboundnoninteractive}
		Let $\mathcal{H}$ be distribution class satisfying \eqref{equ:modelassumption} for $0\leq a \leq 1$. 
		Let data $\{(X_i, y_i)\}_{i=1 }^n$ be generated from \eqref{equ:modelassumption} with $n = n'm'$ and $m = 1$.
		For $0<\varepsilon \leq 1$, let $\beta_{\varepsilon}$ be any $\varepsilon$-LDP estimator of $\beta^*$.
		Then we have
		\begin{align*}
			\inf_{\beta_{\varepsilon}} \sup_{\mathrm{P} \in \mathcal{H}} \mathbb{E}_{\mathrm{P}}\left[ \left\|\beta^*- \beta_{\varepsilon}\right\|_2^2\right] \gtrsim \frac{ds^*}{n'm'\varepsilon^2}. 
		\end{align*}
	\end{proposition}

	The above result yields that for $d \gtrsim nm$, any attempt to LDP sparse linear regression is effortless, as the estimation error does not even converge. 
	In fact, similar negative result also holds when $m$ is small yet larger than 1.

	\begin{proposition}[\textbf{Necessity of sufficiently large $\mathbf{m}$}]\label{lem:necessarym}
		Suppose $s^{*2}\leq n\varepsilon^2 \lesssim \sqrt{d}$ and $m\leq s^*\log d$.  
		Let $\mathcal{H}$ be distribution class satisfying \eqref{equ:modelassumption} with some constant $a \in [0,1]$. 
		Let data $\{(X_i, y_i)\}_{i=1 }^n$ be generated from \eqref{equ:modelassumption}.
		For $0<\varepsilon \leq 1$, let $\beta_{\varepsilon}$ be any $\varepsilon$-LDP estimator of $\beta^*$.
		Then we have
		\begin{align*}
			\inf_{\beta_{\varepsilon}} \sup_{\mathrm{P} \in \mathcal{H}} \mathbb{E}_{\mathrm{P}}\left[ \left\|\beta^*- \beta_{\varepsilon}\right\|_2^2\right]
			\gtrsim \frac{1}{s^*}.
		\end{align*}
	\end{proposition}

	Proposition \ref{lem:necessarym} shows that for $m\leq s^*\log d$, the error does dot converge to zero as $n$ grows. 
	However, this is not the case for user-level LDP if $m$ is sufficiently large. 
	The rigorous counterargument is by establishing upper bound in Theorem \ref{thm:preciseestimation}, which is $\mathcal{O}(s^{*2}  / nm\varepsilon^2)$.
	We explain why the bound fails to generalize. 
	Its proof involves construction of a function class $\mathrm{P}_Z$ and a distribution of $Z$, such that the mutual information between $Z$ and private views $U_1,\cdots, U_n$ is bounded from above and below.
	The former does not hold any more given $m$ samples, since the mutual information becomes larger exponentially in $m$. 
	By carefully bounding the quantity, we establish the following lower bound for ULDP.

	\begin{theorem}[\textbf{ULDP lower bound}]\label{thm:ourlowerbound}
		Suppose $n\varepsilon^2 \geq s^{*2}$, $m\leq d$, and $n\varepsilon^2\leq d$. 
		Let $\mathcal{H}$ be distribution class satisfying \eqref{equ:modelassumption} with $a = \sqrt{\frac{s^*}{m}}$. 
		Let data $\{(X_i, y_i)\}_{i=1 }^n$ be generated from \eqref{equ:modelassumption}.
		For $0<\varepsilon \leq 1$, let $\beta_{\varepsilon}$ be any $\varepsilon$-ULDP estimator of $\beta^*$.
		Then we have
		\begin{align*}
			\inf_{\beta_{\varepsilon}} \sup_{\mathrm{P} \in \mathcal{H}} \mathbb{E}_{\mathrm{P}}\left[ \left\|\beta^*- \beta_{\varepsilon}\right\|_2^2\right] \gtrsim \frac{s^{*2}}{nm \varepsilon^2} .
		\end{align*}
	\end{theorem}

	The result shows that any ULDP estimator admits an error scaling at least with $1 / nm\varepsilon^2$. 
	Thus, a possible improvement for ULDP over LDP lies in replacing $d$ with $s^*$.

	\section{An Algorithm}\label{sec:mainresults}

	
	We begin by outlining our approach to solving the ULDP sparse linear regression problem. 
	A key observation is that with $m$ samples, each user can obtain a rough estimation of parameter with its local samples.
	The central challenge then lies in how to aggregate these rough estimations privately.
	As depicted in Figure \ref{fig:diagram}, our proposed solution operates in two stages.
	In the initial stage, users within the first group independently identify the non-zero elements of $\beta^*$ from their local data and transmit privatized information accordingly.
	By aggregating this information, we determine the $s$ most frequent elements, which serve as the candidate variables for our estimation process.
	On the narrowed parameter space, we estimate the parameter using remaining users. 
	Subsequently, we present the candidate variable selection, final estimation, and extension to general sparse estimations in Section \ref{sec:candidatevariableselection}, \ref{sec:coefestimation}, and \ref{sec:extensiontosparseestimation}, respectively.

	\begin{figure}[htbp]
		\centering
		\includegraphics[width = 0.9\linewidth ]{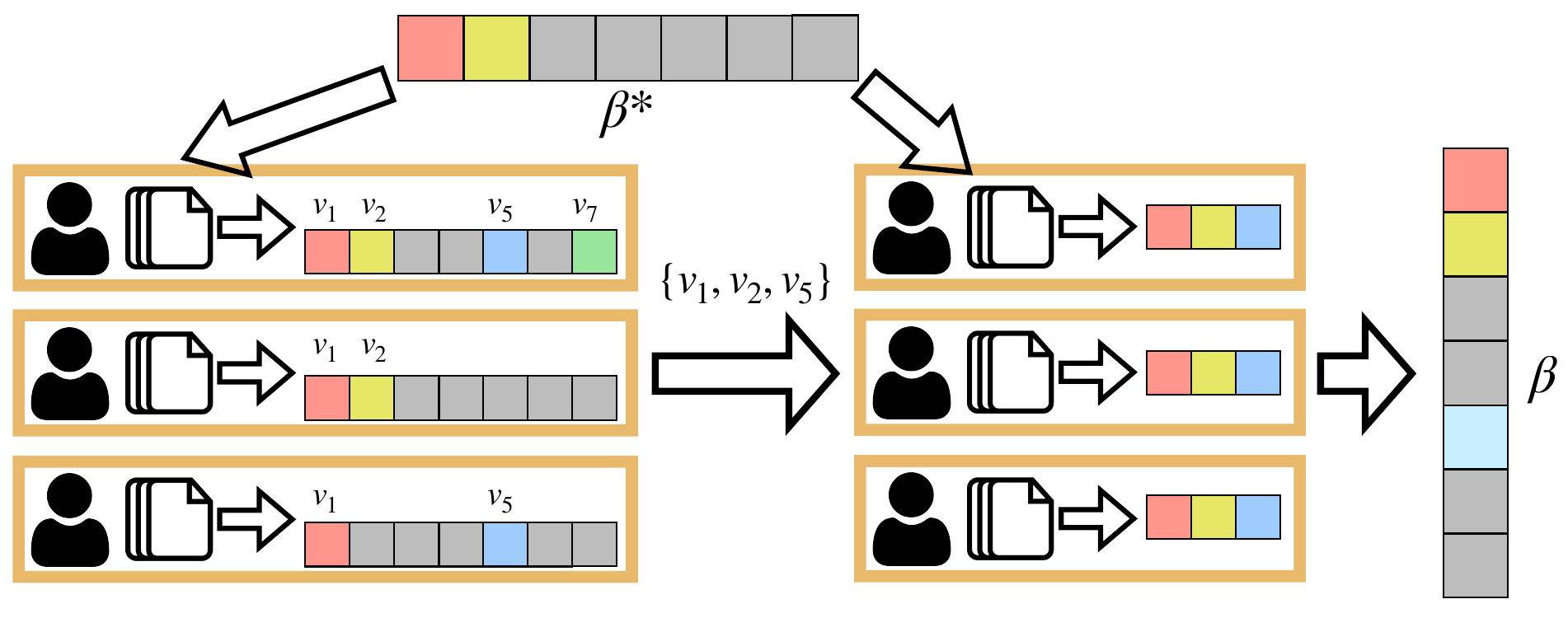}
		\vskip -0.05 in
		\caption{Illustration of the proposed sparse estimation framework.}
		\vskip -0.1 in
		\label{fig:diagram}
	\end{figure}

	\subsection{Candidate Variable Selection}\label{sec:candidatevariableselection}
	
	In this section, we elucidate the steps for candidate variable selection. 
	First, each user prepares a piece of information $v_i \in [d]$, indicating the variable is selected by user $i$ and probably belong to the true variable set. 
	Then, a curator privately aggregates the information ${v}_i$s and outputs the candidate variables.

	To formalize  $v_i$, each user $i$ adopt a local selector $\mathcal{S}_i :( \mathcal{X} \times \mathcal{Y})^m \to [d]$. 
	For each $i$, $\mathcal{S}_i$ can be any plug-in method and is chosen differently based on the constraints of sample size, computational power, and prior information, as long as it produces a good selection results described as follows.

	\begin{definition}[\textbf{$\mathbf{\alpha}$-Good selector}]
		\label{def:successfulscreener}
		Consider user $i$ and its i.i.d. samples $(X_i, y_i) \in ( \mathcal{X} \times \mathcal{Y})^m$ from $\mathrm{P}$.
		For  $0<\alpha< 1$, an $\alpha$-\textit{good selector} is an algorithm $\mathcal{S}$ such that for all $v \in \{1,\cdots, s^*\}$,  there holds
		\begin{align}\label{equ:defofselector1}
			\mathrm{Pr}\left(v = \mathcal{S}\left(X_i, y_i\right)\right)\geq \frac{\alpha}{s^*}.
		\end{align}
		Here, the probability is taken w.r.t. randomness of both $(X_i, y_i)$ and $\mathcal{S}$. 
	\end{definition}

	\eqref{equ:defofselector1} requires a lower bound on probability for the true variables to be selected.
	To induce such selectors, we consider first conducting a local variable selection using $(X_i, y_i)$ and uniformly sampling a $v_i$.
	The following proposition demonstrates that obtaining such a selector is feasible given mild assumptions on distribution $\mathrm{P}$, leveraging well-developed variable selection methods.

	\begin{proposition}[\textbf{Existence of good selectors}]\label{prop:existenceofgoodselectors}
		Under model \eqref{equ:modelassumption}, if either of the following conditions holds, there exists a $\alpha$-good selector with a constant $\alpha$:
		(\textit{i}) $\max_{i\neq j}|\Sigma_{ij}| \leq 3/ s^*$, $a \gtrsim \sqrt{1 / m} $, and $m \gtrsim s^{*2} \log d$;
		(\textit{ii}) $1 \geq a \gtrsim  \sqrt{s^* / m}  \vee \sqrt{\log m \log d / m}$. 
	\end{proposition}

	See Appendix \ref{app:pluginselector} for examples of precise algorithms and detailed proofs.
	(\textit{i}) and (\textit{ii}) are examples of sufficient conditions that are relatively easy to satisfy.
	They require mild correlations among covariates, a strong signal (minimum absolute value of $\beta^*$), and an adequate number of local samples. 
	Similar conditions are standard in high-dimensional statistics \citep{fan2001variable, zhao2006model}.
	Though the lower bound of $a$ is considerable and will leads to a improved minimax rate in the non-private case \citep{ndaoud2019interplay}, it is not the key for ULDP sparse linear regression to be advantages over its LDP counterpart. 
	This is because the function classes constructed in \citet{wang2019sparse, zhu2023improved} for lower bound proof are all covered by the assumptions.
	Additionally, the sample size requirement remains polynomial in $s^*$ and $\log d$, which is theoretically reasonable.

	Given the local information $v_i$, we conduct a private voting to identify the frequently appeared variables $\{\widehat{v}^1, \cdots, \widehat{v}^s\}$. 
	Suppose we use the first $n/2$ users for identification, although the proportion is arbitrary and can be any constant.
	Considering the large size $d$ of the variable universe compared to number of available users, this task is closely related to the problem of heavy hitter detection \citep{bassily2020practical, acharya2021estimating}.
	We solve the identification problem in standard manner \cite{bassily2020practical}, while any tailored approach is adoptable. 
	Specifically, we encode the $d$ variables into a binary string using $\lceil \log d \rceil$ bits.
	Next, we traverse a binary prefix tree from level $1$ to $\lceil \log d \rceil$ and eliminate nodes that cannot serve as prefixes of heavy hitters, namely those with frequencies lower than a certain threshold $\rho$.
	The key advantage of this method is its ability to identify frequent elements with frequencies above $\sqrt{n \log d \log n / \varepsilon^2}$, which overcomes the polynomial dependency on $d$ in LDP discrete density estimation \cite{kairouz2016discrete, duchi2018minimax}. 
	The detailed procedure (\texttt{HeavyHitter}) is provided in Appendix \ref{app:heavyhitter}.

	In the first part of Algorithm \ref{alg:uldpsparselinearregression}, we summarize the pipline for candidate variable selection. 
	The following proposition demonstrates its effectiveness by establishing that, provided the existence of local good selectors, the curator can select a set of variables of size $s\asymp s^*$ containing the true variables.
	This property, known as perfect selection or consistent selection \citep{zhao2006model, belloni2013least}, plays a crucial role in the theoretical properties of subsequent operations.

	\begin{proposition}
		\label{prop:heavyhitterselection}
		Let $\{\widehat{v}_1,\cdots, \widehat{v}_s\}$ be the selected variables in Algorithm \ref{alg:uldpsparselinearregression}. 
		Suppose that all $\mathcal{S}_i$ are $\alpha$-good selectors with $\alpha \gtrsim s^* \sqrt{\log n \log d / n\varepsilon^2} $. 
		If we take $ \alpha / 8 s^* \leq \rho \leq \alpha / 4 s^*$, then with probability $1 - 1 / n^2$, we have (\textit{i}) $\{1,\cdots, s^*\} \subseteq \{\widehat{v}_1,\cdots, \widehat{v}_s\}$,  (\textit{ii}) $s \leq 32 s^* / \alpha  $.
	\end{proposition}

	Note that our scheme samples only one locally selected variable and disregards the others. 
	This select-one-and-aggregate approach has been demonstrated to be as effective as if each user had only one variable \citep{zhu2020federated, cohen2023hot}. 
	To fully utilize the information of the selected variables, we can leverage the set-value heavy hitters \citep{qin2016heavy, zhu2020federated, wang2023locally}. 
	However, this only results in an improvement of $O(\sqrt{s^*})$ in the threshold, which is not our primary focus.

	\subsection{Coefficient Estimation}\label{sec:coefestimation}

	Given the selected variables $\{\widehat{v}_1,\cdots, \widehat{v}_s\}$, the problem is reduced to low dimensional linear regression. 
	Efficient algorithms and fundamental limits have been established \citep{duchi2018minimax, wang2019sparse} for item-level LDP. 
	Leveraging these algorithms, one can ignore all but one sample from each user and obtain an error bound depending polynomially on $s^*$ instead of $d$. 
	However, we would like to explore the benefits brought by having multiple samples per user, as is addressed in the advanced research of ULDP.

	We introduce necessary notations to define the learning problem on the selected subspace. 
	Given Proposition \ref{prop:heavyhitterselection}, we assume the selected variables contain the true ones in the following analysis. 
	Without loss of generality, let $(\widehat{v}_1,\cdots, \widehat{v}_s) = (1, \cdots, s)$.
	We put a hat over the quantities on the selected space. 
	Let $\widehat{\mathrm{P}}$ be the marginal distribution on the selected space 
	$\widehat{\mathcal{X} }\times \mathcal{Y} = \mathbb{R}^{s+1}$, where $\widehat{\mathrm{P}}_{\widehat{X}} = \mathrm{P}_{X^{1:s}}$. 
	Define the data on selected space as $\widehat{X}_{i,j} = {X}_{i,j}^{1:s}$ and $\widehat{X}_{i} = \{{X}_{i,j}^{1:s}\}_{j=1}^m$.
	The underlying coefficients becomes $\widehat{\beta}^* = \beta^{* 1 : s}$. 
	

	\subsubsection{A Multi-round Protocol via SCO}

	At first glance, we can directly find $\beta \in \mathbb{R}^s$ through the following ULDP stochastic convex optimization problem on the selected space
	\begin{align}\label{equ:scoobject}
		\argmin_{\|\widehat{\beta}\|_{\infty} \leq 1} \left(F(\widehat{\beta}) = \int_{\widehat{\mathcal{X}}\times \mathcal{Y}} \left(x^{\top}\widehat{\beta} - y\right)^2 d\widehat{\mathrm{P}}(x,y)\right).
	\end{align}
	Recent study \citep{bassily2023user} provided methodology and established theory with respect to smooth loss functions. 
	We borrow their algorithm, presented in Appendix \ref{app:multiroundprotocol}, which is a private variant of accelerated mini-batch gradient descent.
	It utilize the fact that the gradient of a local batch concentrates with rate $\sqrt{1/ m}$ to reduce the magnitude of noise added to the gradients. 
	While the methodology remains the same, we improve the theoretical analysis in \citet{bassily2023user} to accommodate squared loss, which possesses strong convexity and leads to faster convergence.

	\begin{theorem}[\textbf{Informal}]\label{thm:preciseestimation2}
		Let data $\{(X_i, y_i)\}_{i=1}^n$ be generated as in \eqref{equ:modelassumption}. 
		Suppose $\{\mathcal{S}_i\}_{i=1}^{n/2}$ are $\alpha$-good selectors with $\alpha \gtrsim s^* \sqrt{\log n \log d/ n\varepsilon^2} $. 
		Then with correct parameter choice, solving \eqref{equ:scoobject} leads to an estimation $\beta$ such that
		\begin{align*}
			\mathbb{E}\left[\left\|\beta^* - \beta\right\|_2^2\right]\lesssim  \frac{s^{*9} \log^6 n}{nm\varepsilon^2 \alpha^9}+\frac{s^{*4} \log n }{n m \alpha^4} . 
		\end{align*}
	\end{theorem}
	
	The result stated in Theorem \ref{thm:preciseestimation2} holds \textit{in expectation}, unlike the other conclusions which hold \textit{with high probability}.
	This distinction arises from the formulation of the technical lemma we borrowed.
	Upon initial inspection, we notice that both parts of the theorem involve $\alpha$, indicating a degradation associating to variable selection performance.
	However, according to Proposition \ref{prop:existenceofgoodselectors}, $\alpha$ is merely a constant given a sufficiently large $m$.
	The higher-order term of $s^*$ encompasses various overheads, including the private mean estimation error and the Lipschitz constant of the squared loss over the $\|\cdot\|_{\infty}$ ball.

	\subsubsection{A Two Round Protocol}

	The multi-round protocol is disadvantageous from two perspectives.  
	Firstly, as a gradient-based method, it necessitates $\mathcal{O}(\sqrt{nm\varepsilon^2})$ rounds of communication, which can be prohibitively slow in practice due to network latency \citep{smith2017interaction, zheng2017collect}. 
	Secondly, compared to Theorem \ref{thm:ourlowerbound}, the upper bound provided in Theorem \ref{thm:preciseestimation2} is far from tight concerning $s^*$.
	We question whether these drawbacks can be mitigated for the specific problem of linear regression.
	In this section, we provide an affirmative answer. 
	Our main inspiration stems from the following observation.

	\begin{proposition}\label{prop:distributionofbeta}
		There exists estimators on selected variables $\widehat{\beta}_{n /2 +1},\cdots, \widehat{\beta}_n$, such that for all $\widehat{\beta}_i\in\mathbb{R}^s$, we have $\mathbb{E}_{\mathrm{P}}\left[\widehat{\beta}_i\right] = \widehat{\beta}^*$ and $	\|\widehat{\beta}_i  - \widehat{\beta}^*\|_2 \lesssim \sqrt{{s \log n}/ {m}}$ with probability $1 - 1 / n^2$. 
		Moreover, if either condition in Proposition \ref{prop:existenceofgoodselectors} holds, the bound improves to $	\|\widehat{\beta}_i  - \widehat{\beta}^*\|_2 \lesssim \sqrt{{s^* \log n}/ {m}}$.
		
	\end{proposition}

	Since the mean of $\widehat{\beta}_i$ is $\widehat{\beta}^*$, an ideal estimator would be the mean of $\widehat{\beta}_i$s.
	Moreover, Proposition \ref{prop:distributionofbeta} indicates that $\widehat{\beta}_i$ concentrates as $m$ increases, suggesting that we can confine $\widehat{\beta}_i$ to a restricted area to enhance estimation accuracy.
	We propose a two-stage estimation similar to \citet{girgis2022distributed}.
	First, leveraging user indices $ n / 2 +1 \leq i\leq 3n/4$, we designate a histogram bin on $\mathbb{R}^s$, wherein almost all the $\widehat{\beta}_i$ values will fall.
	Then, the last group of users project their $\widehat{\beta}_i$ onto the bin and add a Laplace noise.
	Given the reduced sensitivity of the projected coefficients, the noise magnitude significantly diminishes.
	We provide detailed methodology (\texttt{ULDPMean}) in Appendix \ref{app:coeffestimaton} and summarize the pipline in Algorithm \ref{alg:uldpsparselinearregression}.

	\begin{algorithm}[htbp]
		\caption{Two-round ULDP sparse estimation.}
		\label{alg:uldpsparselinearregression}
		\begin{algorithmic}
			\STATE	{\bfseries Input: }{ Local data sets $\{(X_i, y_i)\}_{i=1}^n$, selectors $\{\mathcal{S}_i\}_{i=1}^{n/2}$, privacy budget $\varepsilon$,  threshold $\rho$, concentration radius $\tau$.}
			\STATE {\bfseries Initialization: } ${\beta} \in \mathbb{R}^d$ be a zero vector. 
			\STATE {\color{red}\texttt{\# candidate variable selection}}
			\STATE {\color{blue}\texttt{\# on local machine}}
			\FOR{$i$ in $1, \cdots, n/2$}
			\STATE $v_i = \mathcal{S}_i(X_i,y_i)$. 
			\ENDFOR
			\STATE {\color{blue}\texttt{\# $\lceil\log d\rceil$ round communication}}
			\STATE  $\{\widehat{v}_1,\cdots, \widehat{v}_s\}$ = \texttt{HeavyHitter}($\{v_i\}_{i=1}^{n/2}, \varepsilon$, $\rho$).
			\STATE {\color{red}\texttt{\# coefficient estimation}}
			\STATE {\color{blue}\texttt{\# on local machine}}
			\FOR{$i$ in $n/2 + 1, \cdots, n$}
			\STATE Fit $\widehat{\beta}_i$ according to $\left(\widehat{v}_1,\cdots, \widehat{v}_s\right)$.
			\ENDFOR
			\STATE {\color{blue}\texttt{\# 2 round communication}}
			\STATE $\widehat{\beta}$ = \texttt{ULDPMean}($\{\widehat{\beta}_i\}_{i = n /2 + 1}^{3 n / 4}$, $\{\widehat{\beta}_i\}_{i = 3 n /4 + 1}^{n}$, $\tau$, $\varepsilon$). 
			\STATE $\beta^{\widehat{v}_1 : \widehat{v}_s} = \widehat{\beta}$.
			\STATE {\bfseries Output: }{$\beta$}.
		\end{algorithmic}
	\end{algorithm}

	The entire protocol requires a reasonable $\log d + 2$ rounds of communication, with each user sending 1 bit of information. 
	The $\log d$ communication rounds are necessary for \texttt{HeavyHitter}, which can be substituted by any other customized identification method for improved efficiency.
	In the coefficient estimation stage, our method takes two round communication, which is quite efficient.
	Fully utilizing multiple samples necessitates sequential interactivity \citep{acharya2023discrete, bassily2023user}.
	
	We now present the main result, which is the error upper bound of the estimator summarized in Algorithm \ref{alg:uldpsparselinearregression}.

	\begin{theorem}\label{thm:preciseestimation}
		Let data $\{(X_i, y_i)\}_{i=1}^n$ be generated as in \eqref{equ:modelassumption}. 
		Suppose $\{\mathcal{S}_i\}_{i=1}^{n/2}$ are $\alpha$-good selectors with $\alpha \gtrsim s^* \sqrt{\log n \log d/ n\varepsilon^2} $. 
		Suppose we let $\alpha / 8s^* \leq \rho \leq \alpha / 4 s^*$ and $\tau \asymp \sqrt{\log^2 n / m}$. 
		Let ${\beta}$ be the output of Algorithm  \ref{alg:uldpsparselinearregression}. 
		Then we have (\textit{i}) Algorithm \ref{alg:uldpsparselinearregression} is $\varepsilon$-ULDP. (\textit{ii}) there holds
		\begin{align}\label{equ:tworoundbound1}
			\left\|\beta^* - \beta\right\|_2^2 \lesssim   \frac{s^* \log n}{n m \alpha}+ \frac{s^{*2} \log ^3 n}{n m \varepsilon^2 \alpha^2}
		\end{align}
		with probability at least $1 - 4 / n^2$. 
		Moreover, if either condition in Proposition \ref{prop:existenceofgoodselectors} holds, the bound improves to 
		\begin{align}\label{equ:tworoundbound2}
			\left\|\beta^* - \beta\right\|_2^2 \lesssim   \frac{s^* \log n}{n m }+ \frac{s^{*2} \log ^3 n}{n m \varepsilon^2 \alpha}.
		\end{align}
	\end{theorem}

	The upper bound in \eqref{equ:tworoundbound1} consists of two parts. 
	Both terms include additional $ \alpha$s and $\log n$s, which are inevitable due to selection degradation and the overhead of utilizing multiple local samples.
	The $\log n$s are due to the union bound arguments, while $\alpha$ is merely constants by Proposition \ref{prop:existenceofgoodselectors}. 
	Ignoring $\alpha$ and $\log n$,  the first part recovers the rate of non-private linear regression on $\mathcal{O}(s^*)$ dimensional space. 
	The second part corresponds to privacy.
	When $\varepsilon \gtrsim \sqrt{s^* } $, this part is negligible.
	 Algorithm \ref{alg:uldpsparselinearregression} achieves the same error as if its non-private. 
	It worth noting that in most cases (see e.g. Table \ref{tab:summarizetheoreticalresults}), locally private algorithm matches its non-private counterpart when $\varepsilon \gtrsim \sqrt{s^*}$.
	The improvement of \eqref{equ:tworoundbound2} over \eqref{equ:tworoundbound1} is based on the existence of sparse oracles that achieve error $s^* / m$ locally, instead of $s / m$.


	We observe that, unlike common high-dimensional results \cite{wang2019sparse, cai2023private}, our bound does not involve a $\log d$ term. 
	This phenomenon is also noted in \citet{ndaoud2019interplay}, where the $\log d$ disappears if we leverage the beta-min condition in Proposition \ref{prop:existenceofgoodselectors}.
	We will observe in the experiments that if $m$ is large enough, our method is more robust to changes in $d$. 
	However, this is not to say that we can deal with arbitrarily large $d$. 
	The logarithmical relationship is still contained in $\alpha$, which poses a requirement of $m\gtrsim \log d$ as in Proposition \ref{prop:existenceofgoodselectors}.
	Moreover, omitting the log factors, the privacy error is decided by the total number of samples $mn$ for sufficiently large $m$ and $n$. 
	Thus, we can achieve the same estimation error with less number of users if there are more local samples per user, while retaining the same level of privacy for each user since $\varepsilon$ is fixed. 
	On contrary, if there is $mn$ users with one sample each, the error is inevitably $\mathcal{O}(d s^* / nm \varepsilon^2)$ (Proposition \ref{thm:lowerboundnoninteractive}). 
	This comparison illustrates the advantage of having both sufficient users and local samples compared to having abundant users and only one local sample. 
	Note that this distinction holds only between sequential-interactive ULDP and LDP.
	It is unclear whether the lower bound holds under non-interactive ULDP, since most ULDP methods require sequential interactivity \cite{acharya2023discrete, bassily2023user}.

	\subsection{Extension to Sparse Estimation}\label{sec:extensiontosparseestimation}
	
	In this section, we show our framework can be applied to various sparse problems through reduction to non-private learners.
	We consider estimation of $\beta^*$ from data $\{X_i\}_{i=1}^n \in \mathcal{X}^{mn}$, which is generated from distribution $\mathrm{P}_{\beta^*}$ parameterized by $\beta^*$. 
	$\beta^*$ is assumed to be in $\Omega_{s, a}^d$. 
	The assumptions include linear regression as a special case.
	It's important to note that Algorithm \ref{alg:uldpsparselinearregression} depends on the particular problem form via two steps: (\textit{i}) the selector $\mathcal{S}_i$ and (\textit{ii}) the estimator $\widehat{\beta}_i$. 
	Both components depend on a non-private estimator of $\beta^*$.
	The following theorem demonstrates that, given a qualified estimator, our framework achieves fast convergence rates for the general problem of sparse estimation.

	\begin{theorem}[\textbf{Informal}] \label{thm:preciseestimationgeneral}
		Let data $\{X_i\}_{i=1}^n$ be generated by $\mathrm{P}_{\beta^*}$ for $\beta^* \in \Omega_{s,a}^d$.
		Suppose we have non-private estimators:
		(\textit{i})
		estimator $\tilde{\beta}_i$ with $\|\tilde{\beta}_i - {\beta}^*\|_2 \leq \nu_1$ for all $ 1\leq i \leq n / 2$ and 
		(\textit{ii}) 
		estimator $\widehat{\beta}_i$ on selected variables with $\mathbb{E}\left[\widehat{\beta}_i\right] = \widehat{\beta}^*$ and $\|\widehat{\beta}_i - \widehat{\beta}^*\|_2 \leq \nu_2$ for all $n /2 + 1\leq i \leq n$. 
		Then, for any $a \gtrsim\nu_1$,
		there exists an $\varepsilon$-ULDP algorithm whose output $\beta$ has 
		\begin{align}\label{equ:generalbound}
			\left\|\beta^* - \beta\right\|_2^2 \lesssim   \frac{\nu_2^2}{n}+ \frac{\nu_2^2 s^{*} \log^2 n }{n \varepsilon^2\alpha }
		\end{align}
		with probability at least $1 - 3 / n^2$. 
		Moreover, for $\ell_1$ norm, there holds
		\begin{align}\label{equ:generalboundl1}
			\left\|\beta^* - \beta\right\|_1 \lesssim   \sqrt{\frac{\nu_2^2 s^*}{n \alpha}}+ \sqrt{\frac{\nu_2^2 s^{*2 } \log^2 n }{n \varepsilon^2\alpha^2 }}
		\end{align}
		with probability at least $1 - 3 / n^2$. 
	\end{theorem}
	
	We also present a result for the $\ell_1$ norm. 
	Comparing \eqref{equ:generalboundl1} to \eqref{equ:generalbound}, the difference arises from the $\sqrt{s}$ discrepancy between the $\ell_1$ and $\ell_2$ norms, given that we only have $s$ non-zero elements in our sparse estimation problem.
	We discuss the implications of Theorem \ref{thm:preciseestimationgeneral}. 
	Consider the sparse mean estimation \citep{duchi2018minimax, zhou2022locally}, where non-private estimator achieves $\nu_2 = \mathcal{O}(\sqrt{s^* \log n / m})$ \citep{johnstone1994minimax} under mild conditions. 
	Then the bound \eqref{equ:generalbound} becomes identical to \eqref{equ:tworoundbound2}, which eliminates the linear dependency of $d$ in LDP \citep{duchi2018minimax}. 
	For sparse discrete distribution estimation, \citet{acharya2021estimating} removed the linear dependency of $d$. 
	With $\nu_2 = \mathcal{O}(\sqrt{s^*\log n / m})$, our bound \eqref{equ:generalboundl1} is $\sqrt{s^*}$ larger than theirs in $\ell_1$ sense.

	It worth mentioning that when $d$ is small, our upper bound matches the lower bound for $m = 1$. 
	In this scenario, selector provides no useful information and is equivalent to a random selection, i.e. $\alpha \leq \mathrm{P}\left(v = \mathcal{S}(X_i, y_i)\right) \cdot s^* = s^* / d $.
	If $\alpha = s^* / d  \gtrsim s^* \sqrt{\log n \log d/ n\varepsilon^2}$, then \eqref{equ:tworoundbound2} becomes
	\begin{align*}
		\frac{s^*   \log n}{nm} + \frac{ds^* \log^3 n}{n m \varepsilon^2 }. 
	\end{align*}
	Up to logarithmic factors, the second term matches the lower bound established in \citet{zhu2023improved} for sparse linear regression and \citet{duchi2018minimax} for sparse mean estimation.

	\section{Experiment Results}\label{sec:experiments}

	We conduct experiments on both synthetic and real datasets to show the superiority of proposed methods and to validate our theoretical findings. 
	The tested methods include:
	\textbf{(\textit{i}) 2-SLR}: The proposed two-round ULDP sparse linear regression method outlined in Algorithm \ref{alg:uldpsparselinearregression}; 
	\textbf{(\textit{ii}) M-SLR}: The proposed multi-round version in Algorithm \ref{alg:uldpsparselinearregressionmultiround}. 
	The competing methods are:
	\textbf{(\textit{iii}) LDPPROX}: The non-interactive LDP proxy estimator in \citet{zhu2023improved};
	\textbf{(\textit{iv}) LDPIHT}: The LDP iterative hard thresholding in \citet{wang2019sparse, zhu2023improved}.
	Both comparison methods receive $nm$ samples with budget $\varepsilon$ each. 
	Additionally, we report performance of non-privately fitting \textbf{(\textit{v}) Lasso} using $m$ samples, representing an alternative for each user to rely solely on their local information. 
	Implementation details are provided in Appendix \ref{app:additionalexperiment}.
	For each model, we report the best result over its parameter grids, with the best result determined based on the average of at least 30 replications.
	The size of the parameter grids is selected based on running time to ensure that each method incurs an equal amount of computation.
	All experiments are conducted on a machine with 72-core Intel Xeon 2.60GHz and 128GB of main memory. 
	The code is publicly available at GitHub\footnote{https://github.com/Karlmyh/ULDP-SL}.

	\begin{figure*}[!t]
		\vskip -0.1in
		\centering
		\subfigure[$d$ - F1 score.]{
			\begin{minipage}{0.22\linewidth}
				\centering
				\includegraphics[width=\linewidth]{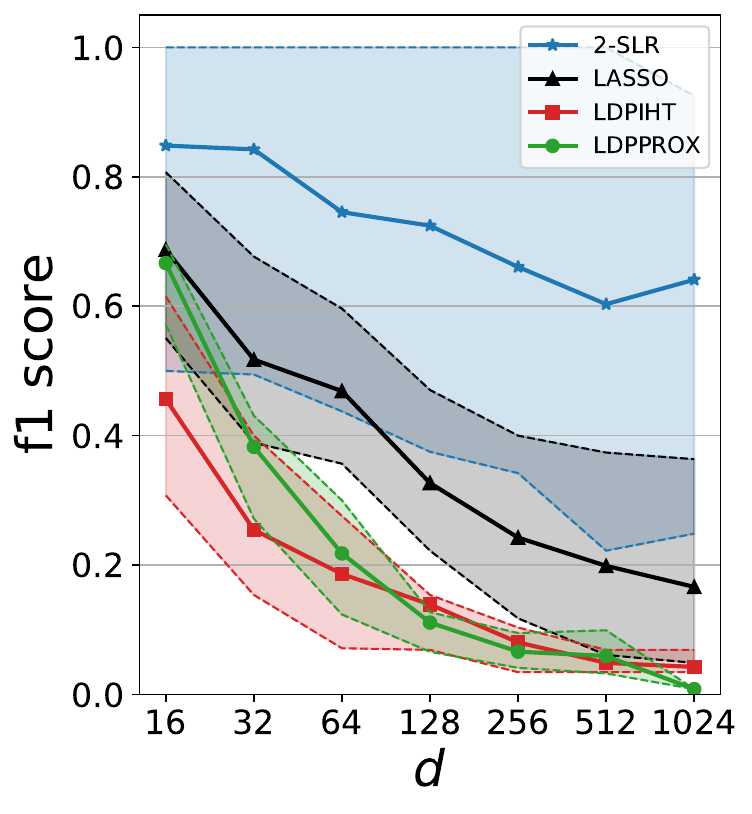}
			\end{minipage}
			\label{fig:df1}
		}
		\subfigure[ $d$ - $\ell_2$ error with $m = 100$.]{
			\begin{minipage}{0.22\linewidth}
				\centering
				\includegraphics[width=\linewidth]{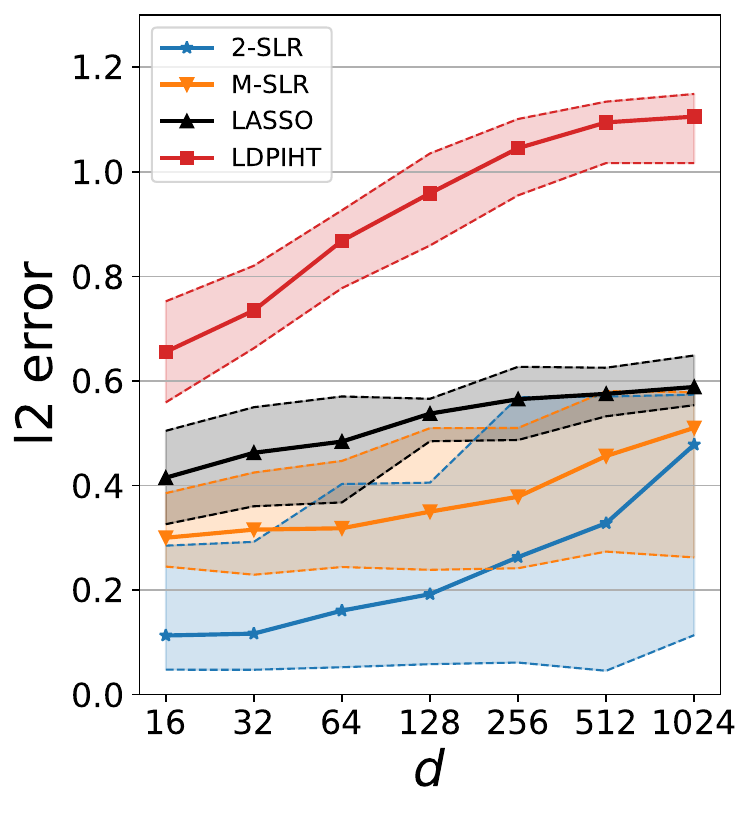}
			\end{minipage}
			\label{fig:dl2m100}
		}
		\subfigure[$d$ - $\ell_2$ error with $m = 200$.]{
			\begin{minipage}{0.22\linewidth}
				\centering
				\includegraphics[width=\linewidth]{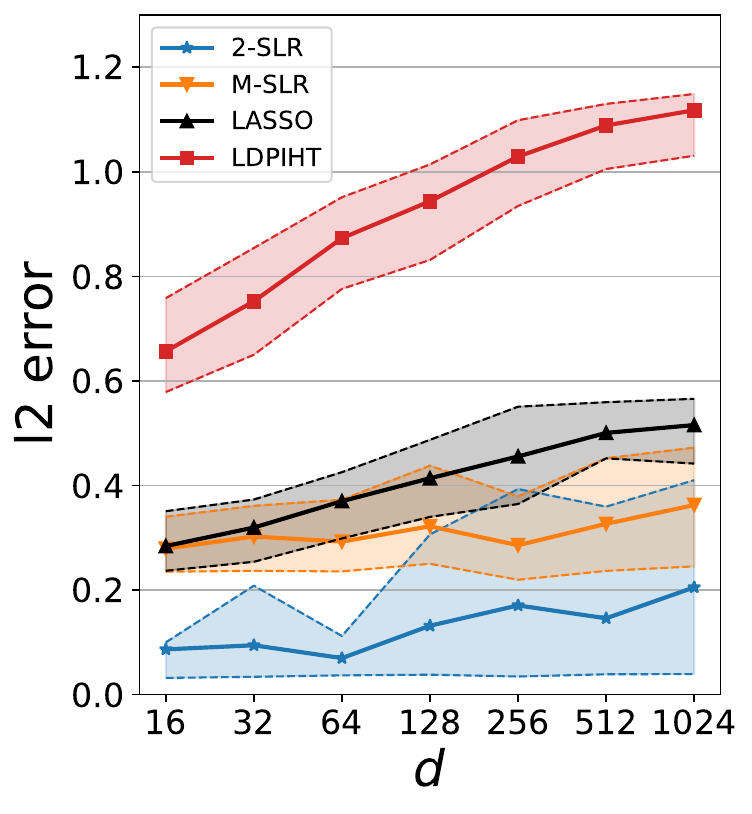}
			\end{minipage}
			\label{fig:dl2m200}
		}
		\subfigure[$\varepsilon$ - $\ell_2$ error.]{
			\begin{minipage}{0.22\linewidth}
				\centering
				\includegraphics[width=\linewidth]{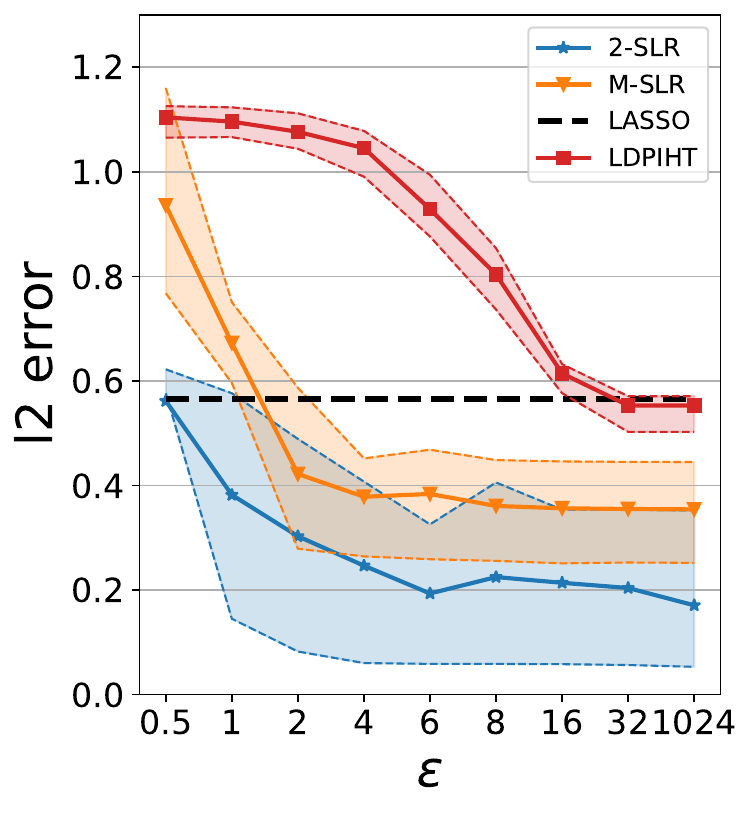}
			\end{minipage}
			\label{fig:epsilonl2}
		}
		\vskip -0.1in
		\caption{Experiments w.r.t. $d$ and $\varepsilon$. 
			We plot the quantiles over 30 repetitions with $95\%$ coverage. 
			We exclude LDPPROX in the last three figures since it is highly unstable and do not fit into our plot scale. 
		}
		\vskip -0.14in
		\label{fig:epsilond}
	\end{figure*}
	
	\subsection{Simulation}

	We conducted experiments on synthetic data to validate the theoretical findings.
	Two sets of parallel experiments are conducted for independent and correlated marginal distributions, respectively, while results of the latter are presented in Appendix \ref{app:additionalexperiment}. 
	We draw each $X_{i,j}^k$ and $\sigma_{i,j}$ independently from standard Gaussian distribution.
	For $\beta^*$, we randomly select $s^* = 8$ coordinates to be $0.2$ and let others be zero.
	Typically, we set $n = 400$, $m = 100$, $d = 256$, and $\varepsilon = 4$, while varying one of them to observe how the evaluated metric varies.
	We use squared error to evaluate the estimated coefficients and F1 score to evaluate variable selection.

	We conduct experiments w.r.t. $d$. 
	We first analyze the variable selection performance.
	For $d\in\{16,32, \cdots, 1024\}$, we compute the averaged F1 scores of the proposed candidate variable selection (represented by 2-SLR) and other methods. 
	As shown in Figure \ref{fig:df1}, the selection performance of 2-SLR is superior to variables induced by other methods.
	Particularly noteworthy is that 2-SLR achieved higher F1 scores than Lasso.
	This observation aligns with  \citet{wang2011random, liang2023vsolassobag}, where aggregating Lasso fitted on random subsamples leads to performance gains in both selection and prediction.
	
	Next, we analyze the estimation performance with respect to $d$. 
	In Figures \ref{fig:dl2m100} and \ref{fig:dl2m200}, we plot the curve of $\ell_2$ error w.r.t. $d$. 
	Given a large $m = 200$, the proposed methods are less sensitive to $d$ compared to LDPIHT and Lasso.
	This observation is compatible with the rate in \eqref{equ:tworoundbound2}, which is independent of $d$. 
	Conversely, for smaller $m = 100$, the local selectors can not provide a constant $\alpha$ for exponentially larger $d$.
	As a result, the trend of our methods is steeper.

	We examine the privacy-utility trade-offs by investigating performances under different $\varepsilon$. 
	In Figure \ref{fig:epsilonl2}, the error decreases as $\varepsilon$ increases for all private methods. 
	Moreover, the error of 2-SLR is consistently better than Lasso, while error of M-SLR quickly drops below Lasso at medium privacy levels ($\varepsilon \geq 2$).
	This shows the superiority of our methods compared to fitting Lasso using only local information.

		\begin{figure}[!b]
		\centering
		\vskip -0.1in
		\subfigure[$n  = 100$ ]{
			\begin{minipage}{0.3\linewidth}
				\centering
				\includegraphics[width=\linewidth]{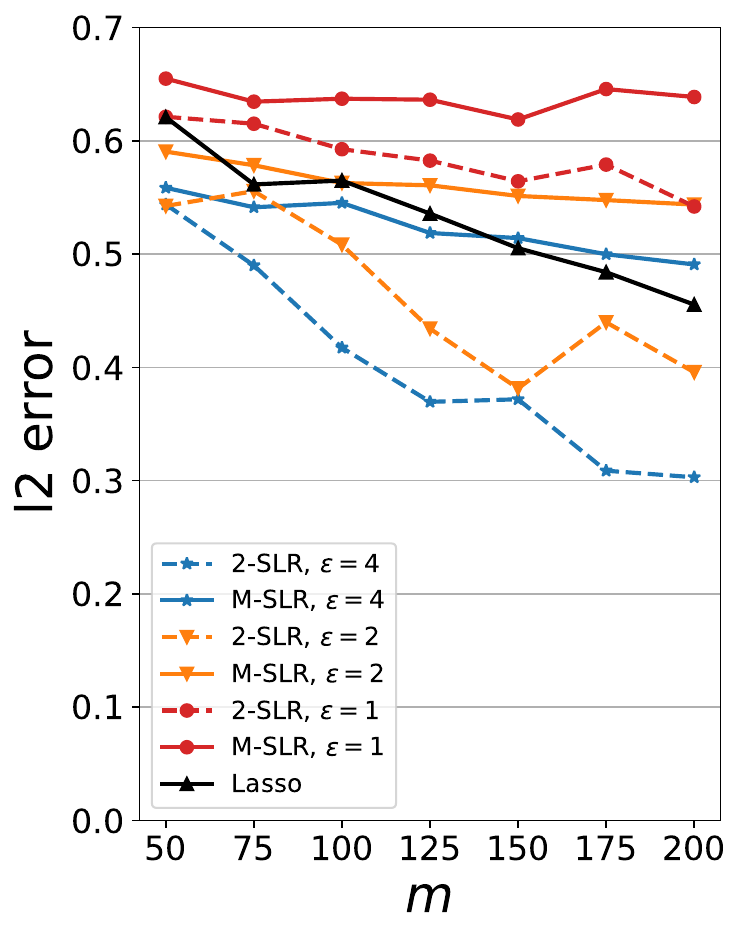}
			\end{minipage}
			\label{fig:mworse}
		}
			\hskip -0.1in
		\subfigure[$n  = 400$ ]{
			\begin{minipage}{0.3\linewidth}
				\centering
				\includegraphics[width=\linewidth]{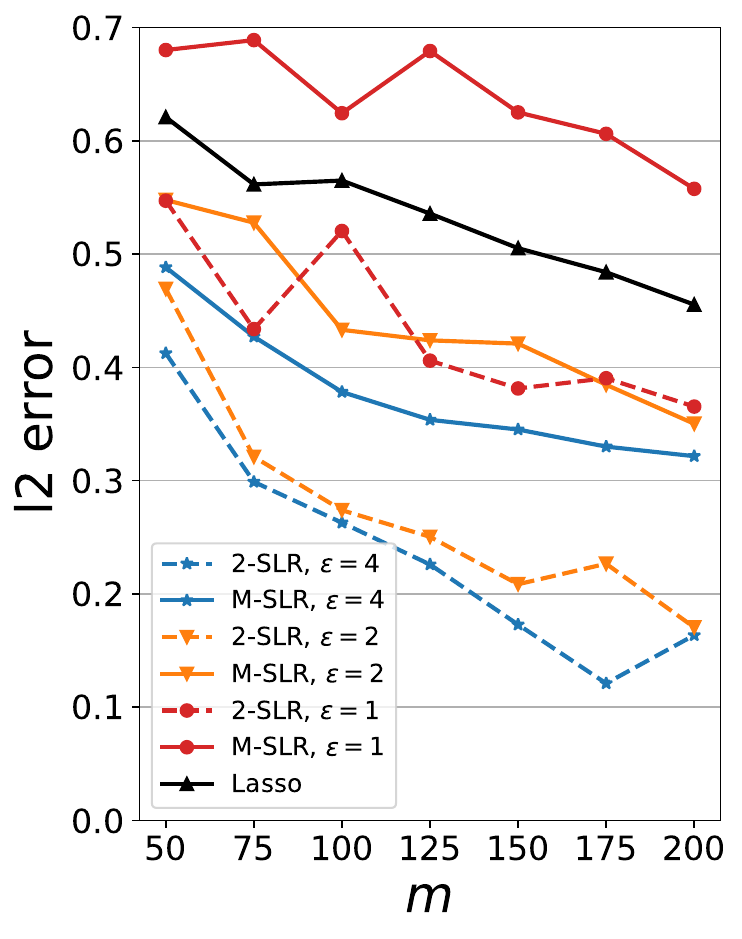}
			\end{minipage}
			\label{fig:m}
		}
			\hskip -0.1in
		\subfigure[$n  = 800$ ]{
			\begin{minipage}{0.3\linewidth}
				\centering
				\includegraphics[width=\linewidth]{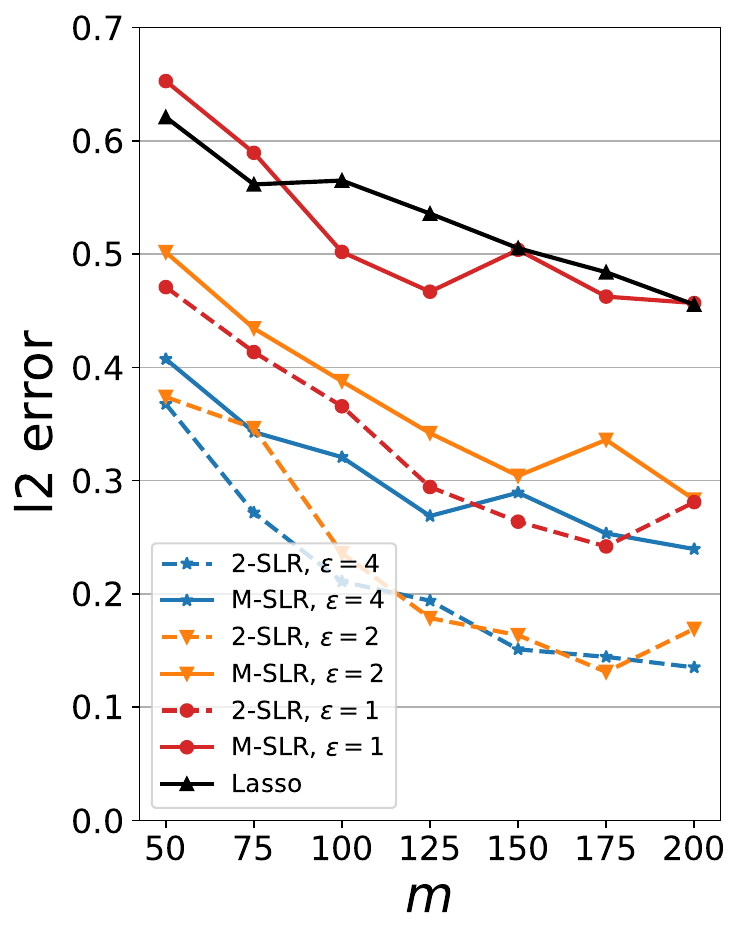}
			\end{minipage}
			\label{fig:mbetter}
		}
		\caption{Experiments w.r.t. $m$ and $\ell_2$ error. 
		}
		\vskip -0.1in
		\label{fig:mall}
	\end{figure}
	
		Finally, we analyze the impact of sample sizes. We conducted experiments with varying $m$ (ranging from 50 to 200) under different $n$, comparing the performance of our methods with Lasso on local samples. The results for varying $m$ are presented in Figure \ref{fig:mall}. We observe that given a sufficiently large $n = 800$, 2-SLR always outperforms Lasso, and M-SLR performs comparably even for $\varepsilon = 1$. If $n = 400$, only M-SLR with $\varepsilon = 1$ performs worse than Lasso. However, given an insufficient $n = 100$, Lasso performs comparably to 2-SLR with $\varepsilon = 2$. Similarly, in Figures \ref{fig:nl2} and \ref{fig:ml2}, the $\ell_2$ error decreases as $n$ increases for all $\varepsilon$. The results indicate that our methods outperform Lasso under various $(n, m, \varepsilon)$ settings, except for M-SLR with $\varepsilon = 1$.
		This observation is reasonable and aligns with phenomena commonly observed in ULDP learning, federated learning, or transfer learning, where incorporating information from other data sources may not necessarily improve estimation if the quality of that additional information is low due to factors such as privacy constraints, data heterogeneity, or data compression.
		
		Moreover, we set $nm = 400 \times 100$ and varied the ratio $n / m$. In Figure \ref{fig:nml2}, we observe that, for each $\varepsilon$, the error of 2-SLR remains stable when $n \approx m$, while it slightly increases when either $n$ or $m$ is too small, which is consistent with Theorem \ref{thm:preciseestimation}. Furthermore, the performance of M-SLR is more sensitive to $n$ becoming small. This is attributed to its gradient nature, which requires a large number of users.

	\begin{figure}[!b]
		\vskip -0.1in
		\centering
		\subfigure[$n$ - $\ell_2$ error.]{
			\begin{minipage}{0.31\linewidth}
				\centering
				\includegraphics[width=\linewidth]{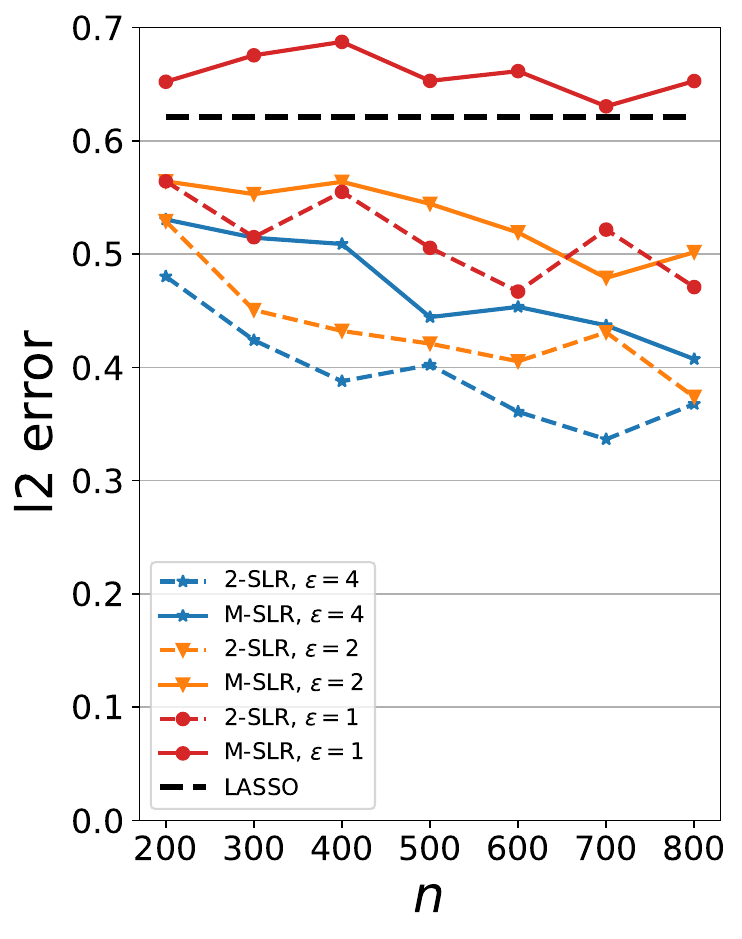}
			\end{minipage}
			\label{fig:nl2}
		}
		\hskip -0.1in
		\subfigure[$m$ - $\ell_2$ error.]{
			\begin{minipage}{0.31\linewidth}
				\centering
				\includegraphics[width=\linewidth]{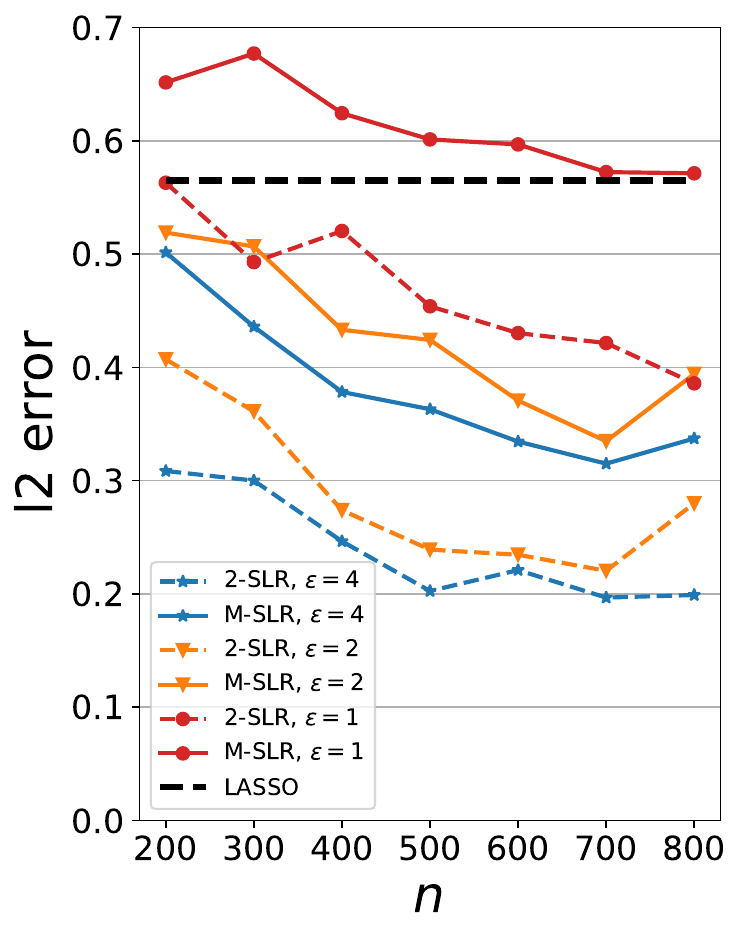}
			\end{minipage}
			\label{fig:ml2}
		}
		\hskip -0.1in
		\subfigure[$n / m$ - $\ell_2$ error.]{
			\begin{minipage}{0.31\linewidth}
				\centering
				\includegraphics[width=\linewidth]{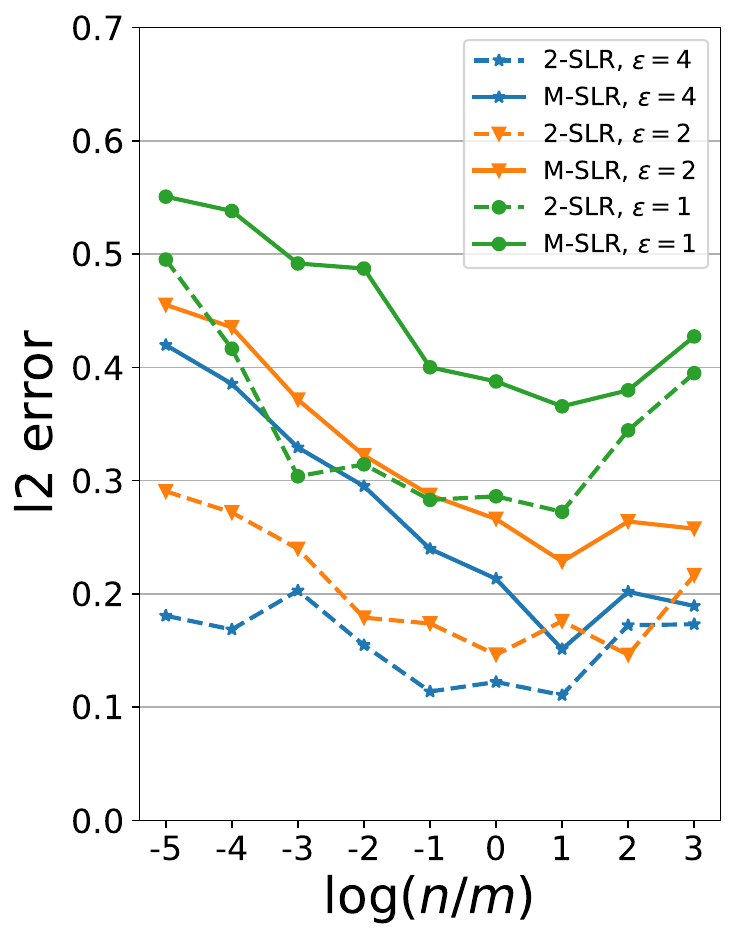}
			\end{minipage}
			\label{fig:nml2}
		}
		\vskip -0.1in
		\caption{Experiments w.r.t. $n$ and $n / m$. 
		}
		\vskip -0.1in
		\label{fig:nm}
	\end{figure}

	\subsection{Real Data}
	
		\begin{table*}[!t]
		\vskip -0.2in
		\caption{Real data performances.  
			To ensure significance, we employ the Wilcoxon signed-rank test \citep{wilcoxon1992individual} with a significance level of 0.05 to determine if a result is significantly better. 
			The best results are \textbf{bolded} and those holding significance towards the rest results are marked with $*$.}
		\label{tab:realdata}
		\centering
		\resizebox{0.75\linewidth}{!}{
			\renewcommand{\arraystretch}{1}
			\setlength{\tabcolsep}{3pt}
			\begin{tabular}{|l|l|ll|l|ll|ll|}
				\toprule
				Budget                            & Datasets   & NP-2-SLR&    NP-M-SLR   & Lasso & 2-SLR & M-SLR & LDPPROX & LDPIHT \\ \midrule
				\multirow{6}{*}{$\varepsilon = 1$} 
				& \texttt{Airline}  & 1.01 &  0.82& 1.02    & 1.02     & \textbf{0.98*}     & 1.38       & 1.85      \\
				& \texttt{Loan}   & 0.97&   0.88& 0.99     & 0.98    & \textbf{0.97}     & 5.27       & 2.00     \\
				& \texttt{MIP}   & 1.00 &  0.96   & 1.65     & 1.00     & \textbf{0.98*}     & 2.54       & 1.87      \\
				& \texttt{Taxi}  & 0.95&   0.01  & 1.04     &  0.96    & \textbf{0.01*}     & 1.20       & 1.02    \\
				& \texttt{Wine}   & 1.19 &  1.17  & \textbf{1.14*}    & 1.34     & 1.37     & 7.71       & 2.30      \\
				& \texttt{Yolanda}  & 1.10 &1.14  & \textbf{1.19}     & \textbf{1.19}     & 1.22     & 1.90       & 2.36     \\ \midrule
				\multirow{6}{*}{$\varepsilon = 4$} 
				& \texttt{Airline}  &1.01 &  0.82 & 1.02     & 1.02     & \textbf{0.88*}     & 1.15       & 1.02      \\
				& \texttt{Loan}  &0.97 &0.88  & 0.99     & 0.98     & \textbf{0.90*}     & 2.05       & 1.65     \\
				& \texttt{MIP}   & 1.00 & 0.96  & 1.65     & 1.01     & \textbf{0.96*}     & 3.30       & 1.82      \\
				& \texttt{Taxi}   &0.95 & 0.01  & 1.04     & 0.95     & \textbf{0.01*}     & 1.16       & 1.88     \\
				& \texttt{Wine}    & 1.19& 1.17  & \textbf{1.14*}     &   1.19   & 1.27     & 5.39       & 1.74      \\
				& \texttt{Yolanda}   & 1.10& 1.14 & 1.19     & \textbf{1.11*}    & 1.18     & 1.79       & 2.03     \\ \midrule
				\multicolumn{2}{|c|}{Rank sum}   &  \multicolumn{2}{|c|}{-} &  31    &  24   &   \textbf{19}  &  56     &   50   \\ 
				\bottomrule
			\end{tabular}
		}
		\vskip -0.15in
	\end{table*}
	
	We conduct experiments on six real datasets with various sample sizes and dimensionalities. 
	Among the datasets, \texttt{Airline} and \texttt{Taxi} are the most suitable for our setting, where each user possesses small local samples with large dimensions.
	The datasets contain sensitive information and have been used in privacy research \citep{ma2024decision}. 
	The other datasets are manually grouped to fit our framework.
	See Appendix \ref{app:realdatasets} for description of datasets.

	We first compute the mean squared error over 30 random train-test splits for $\varepsilon = 1$ and $\varepsilon = 4$.
	To standardize the scale across datasets, we report the MSE ratio relative to non-private fitting with Lasso over all samples.
	The results are displayed in Table \ref{tab:realdata}.
	For both high privacy ($\varepsilon = 1$) and medium privacy ($\varepsilon = 4$), the proposed methods significantly outperform competitors in terms of both average performance (rank sum) and the number of best results achieved.
	It is worth noting that in most cases, Lasso fitted on local datasets outperforms LDP competitors, yielding the effortlessness of LDP sparse regression. 
	Moreover, the running time of the methods is displayed in Appendix \ref{app:realdatasets}.
	The results show that, if properly paralleled, our methods are quite efficient.

		We observe that our methods (2-SLR and M-SLR) can sometimes outperform non-private Lasso on the whole data. This is somewhat expected. As explained in previous literature \citep{ndaoud2019interplay}, given strong signal strength ($\min_{\beta^j>0} |\beta^j|$ is large), the optimal error can actually be improved and simply performing Lasso does not achieve this optimality. 
	Moreover, methodological works \citep{wang2011random, liang2023vsolassobag} showed the effectiveness of selecting candidate variables by aggregating Lasso fitted on random subsamples. Intuitively, even with strong signal strength, fitting Lasso does not guarantee the selection of all true variables due to randomness, while aggregating variables selected on random subsamples is more likely to identify true variables.
	We validated our conjecture by running non-private SLRs ($\varepsilon = 1024$). The results are presented in Table \ref{tab:realdata}. We observe that 2-SLR and M-SLR never outperform their non-private counterparts, while non-private SLRs occasionally outperform Lasso on some datasets.

	We also observe that 2-SLR outperforms M-SLR in simulation, while the opposite is true in real data. 
		The phenomenon is attributed to the implicit regularization.
		 In synthetic data, where the data is neatly generated, estimations tend to converge well. However, real data often contains more noise, leading to potentially unstable estimations. In such cases, using zero coefficients as the initial point yields a regularized estimator \citep{ali2019continuous}, which are biased yet stable.

	\section{Discussion}
	
	In this work, we investigate the ULDP sparse linear regression.
	By proposing a two-phase solution, we show the theoretical advantage of having multiple samples per user, which is then validated by exhaustive experiments.
	
	It is worth mentioning that we do not explore scenarios where $m$ is small, such as $m \leq s^* \log d$.
	Our experiments, particularly with the \texttt{MIP} dataset, demonstrate that even with few local samples, satisfactory results can be achieved.
	However, dealing with small $m$ may require a comprehensive distributional analysis of variable selection, which could be a promising avenue for future research.
	We also hope to establish a tight minimax lower bound of sparse estimation under ULDP. 
	
	Currently, we consider a support estimation based algorithm. 
	As suggested by the reviewers, an interesting topic would be an algorithm that simultaneously learns the sparse coefficients and optimizes the model, potentially with a Lasso-type optimization objective.
	Directly solving such a problem is ineffective \citep{bassily2023user}.
	Each update step will involve updating $d - s$ redundant parameters, whose information needs to be protected under differential privacy.
	Thus, excessive random noise is injected.
	By utilizing support estimation, we circumvent this issue in the second phase, leading to improved final rates.
	A private analog for algorithms with limited message passing each round is promising, such as least-angle regression (LARS) or coordinate gradient descent. 
	
	\section*{Impact Statement}
	
	We believe that it is difficult to clearly foresee societal consequence of the present work, which has a primary focus on machine learning theory and methodology. 
	We believe this work can serve as a forward step to enclosing the gap between the theoretical study of LDP  and  practical situations. 
	
	\section*{Acknowledgement}
	
	We would like to thank the reviewers for their help and advice, which led to a significant improvement of the article.
	We also thank Yifan Gu for providing discussion on variable selection issues. 
	The research is supported by the Special Funds of the National Natural Science Foundation of China (Grant No. 72342010). 
	Yuheng Ma is supported by the Outstanding Innovative Talents Cultivation Funded Programs 2023 of Renmin University of China.
	This research is also supported by Public Computing Cloud, Renmin University of China.

	\bibliography{references}
	\bibliographystyle{icml2024}

	\newpage
	\appendix
	\onecolumn

	In this appendix, we provide the omitted content for minimax lower bound (Appendix \ref{app:lowerbounds}), the algorithm and theoretical results of candidate variable selection  (Appendix \ref{app:candidatevariableselection}), the algorithm and theoretical results of coefficient estimation (Appendix \ref{app:coeffestimaton}), an extension from our framework to general problems (Appendix \ref{app:extensiontosparseestimation}), and details as well as additional results of experiments (Appendix \ref{app:additionalexperiment}).

	%

	\section{Minimax Lower Bound}\label{app:lowerbounds}

	We first borrow assumptions and definitions from \citet{acharya2020unified}.
	Let $Z=\left(Z_1, \ldots, Z_d\right)$ be a random variable over $\mathcal{Z}=\{-1,+1\}^d$ such that $\mathbb{P}\left[Z_i=1\right]=\tau$ for all $i \in[d]$ and the $Z_i$ s are all independent; we denote this distribution by $\operatorname{Rad}(\tau)^{\otimes d}$. For $z \in \mathcal{Z}$, we denote $z^{\oplus i} \in \mathcal{Z}$ as the vector obtained by flipping the sign of the $i$-th coordinate of $z$.
	
	\begin{condition}\label{asp:densityexist}
		For every $z \in \mathcal{Z}$ and $i \in[d]$ it holds that $\mathrm{P}_{z^{\oplus i}} \ll \mathrm{P}_z$ (we refer to $\mathrm{P}_{\beta_z}$ simply as $\mathrm{P}_z$), and there exist measurable functions $\phi_{z, i}: \mathbb{R}^d \rightarrow \mathbb{R}$ such that
		\begin{align*}
			\frac{\mathrm{d} \mathrm{P}_{z^{\oplus i}}}{\mathrm{~d} \mathrm{P}_z}=1+\phi_{z, i} .
		\end{align*}
	\end{condition}
	
	\begin{condition}\label{asp:orthogonality}
		There exists some $\alpha^2 \geq 0$ such that, for all $z \in \mathcal{Z}$ and distinct $i, j \in$ $[d], \mathbb{E}_{\mathrm{P}_z}\left[\phi_{z, i} \cdot \phi_{z, j}\right]=0$ and $\mathbb{E}_{\mathrm{P}_z}\left[\phi_{z, i}^2\right] \leq \alpha^2$.
	\end{condition}

	\begin{condition}\label{asp:addtive}
		For every $z, z^{\prime} \in \mathcal{Z}=\{-1,+1\}^d$,
		\begin{align*}
			\ell_2\left(\theta_z, \theta_{z^{\prime}}\right)=4 \nu\left(\frac{\mathrm{d}_{\mathrm{Ham}}\left(z, z^{\prime}\right)}{\tau d}\right)^{1 / 2}
		\end{align*}
		where $\mathrm{d}_{\mathrm{Ham}}\left(z, z^{\prime}\right):=\sum_{i=1}^d \eins \left\{z_i \neq z_i^{\prime}\right\}$ denotes the Hamming distance, where $\tau=s^* / 2 d, s^*$ and $\nu$ denotes sparsity and error rate respectively.
	\end{condition}

	\begin{proof}[\textbf{Proof of Theorem \ref{thm:ourlowerbound}}]
		First, suppose $X^j$ is uniformly distributed on $\{-1,1\}$ for $1\leq j \leq d$. 
		Let 
		\begin{align*}
			\beta_{Z, j}^* = \frac{4 \sqrt{2 } \nu }{\sqrt{ s^*}}\frac{Z_j + 1}{2}
		\end{align*}
		for $1 \leq j \leq d$ where $Z_j$s are i.i.d. random variables with 
		\begin{align*}
			\mathrm{Pr}\left[Z_i=+1\right]=\frac{s^*}{2 d}, \quad \mathrm{Pr}\left[Z_i=-1\right]=1-\frac{s^*}{2 d} .
		\end{align*}
		There holds $\beta^*_Z$ satisfies the conditions that $\|\beta^*_Z\|_{\infty} \leq 1$ and $\|\beta^*_Z\|_0 \leq s^*$ with probability $1 - s^* / 2d$ using Fact 1 in \citet{acharya2020unified}. 
		Next, for each $Z$ we let:
		\begin{align*}
			\sigma_Z=\left\{\begin{array}{lll}
				1-\left\langle X, \beta^*_Z\right\rangle & \text { w.p. } & \frac{1 + \left\langle X, \beta^*_Z\right\rangle}{2} \\
				-1-\left\langle X, \beta^*_Z\right\rangle & \text { w.p. } & \frac{1 -\left\langle X, \beta^*_Z\right\rangle}{2}
			\end{array}\right.
		\end{align*}
		Thus, $Y \in \{-1, 1\}$. 
		The above distribution satisfies \eqref{equ:modelassumption} with probability $1 - s^* / 2d$.
		The distribution $\mathrm{P}_Z$ has density function $\left(1+Y\left\langle X, \beta^*_Z\right\rangle\right) / {2^{d+1}}$ for $(X, Y) \in\{+1,-1\}^{d+1}$. 
		Then, for the $i$-th user who has the data sample $\left(X_i, y_i\right)$ from the distribution $\mathrm{P}_{Z}^m$, it sends its information through a private algorithm $\mathcal{S}$ after getting messages $S_1,\cdots, S_{i-1}$.
		By definition, for $1\leq j \leq m$, we have 
		\begin{align}\label{equ:lowerboundratioofP}
			\frac{d \mathrm{P}_{z^{\oplus k}}}{d \mathrm{P}_z}  = \prod_{j=1}^m\frac{1+y_{i,j}\left\langle X_{i,j}, \beta_{z^{\oplus k}}\right\rangle}{1+y_{i,j}\left\langle X_{i,j}, \beta_z\right\rangle}=
			\prod_{j=1}^m 1+\frac{y_{i,j}\left\langle X_{i,j}, \beta_{z^{\oplus k}}-\beta_z\right\rangle}{1+y_{i,j}\left\langle X_{i,j}, \beta_z\right\rangle}= 
			\prod_{j=1}^m 1 - \frac{y_{i,j} X_{i,j}^k z_k }{1+y_{i,j}\left\langle X_{i,j}, \beta_z\right\rangle} \cdot \frac{4\sqrt{2  } \nu}{\sqrt{ s^*}}
		\end{align}
		where the last step follows from \citet{zhu2023improved}. 
		If we let $\nu$ to be small enough, we can guarantee that $|y_{i,j}\langle X_{i, j}, \beta_z\rangle| \leq 1 / 2$ for each $z$ and $|\frac{y_{i,j} X_{i,j}^k z_k }{1+y_{i,j}\left\langle X_{i,j}, \beta_z\right\rangle} \cdot \frac{4\sqrt{2  } \nu}{\sqrt{ s^*}}| \leq 1 / 2$. 
		We compute the $\log $ transformation of the above quantity which is $\sum_{j=1}^m\log \left(1 - \frac{y_{i,j} X_{i,j}^k z_k }{1+y_{i,j}\left\langle X_{i,j}, \beta_z\right\rangle} \cdot \frac{4\sqrt{2  } \nu}{\sqrt{ s^*}}\right) $.
		For each $j$, we bound the expectation by Jensen's inequality
		\begin{align}\label{equ:lowerboundlogexpansion}
			\mathbb{E} \left[	\log \left(1 - \frac{y_{i,j} X_{i,j}^k z_k }{1+y_{i,j}\left\langle X_{i,j}, \beta_z\right\rangle} \cdot \frac{4\sqrt{2  } \nu}{\sqrt{ s^*}}\right)\right] \leq  	\log \left(1 -\mathbb{E} \left[ \frac{y_{i,j} X_{i,j}^k z_k }{1+y_{i,j}\left\langle X_{i,j}, \beta_z\right\rangle}\right] \cdot \frac{4\sqrt{2  } \nu}{\sqrt{ s^*}}\right). 
		\end{align}
		For each $k$, we have 
		\begin{align}\label{equ:lowerboundanyt}
			\left|\mathbb{E} \left[ \frac{y_{i,j} X_{i,j}^k z_k }{1+y_{i,j}\left\langle X_{i,j}, \beta_z\right\rangle}\right] \right|\leq \left|\frac{1}{2 + 8 \sqrt{2 s^*} \nu } - \frac{1}{2 - 8 \sqrt{2 s^*} \nu }\right| \leq  8\sqrt{2s^*} \nu.  
		\end{align}
		Bringing \eqref{equ:lowerboundanyt} into \eqref{equ:lowerboundlogexpansion} leads to 
		\begin{align*}
			\mathbb{E} \left[	\log \left(1 - \frac{y_{i,j} X_{i,j}^k z_k }{1+y_{i,j}\left\langle X_{i,j}, \beta_z\right\rangle} \cdot \frac{4\sqrt{2  } \nu}{\sqrt{ s^*}}\right)\right] \leq \log \left(1 +64 \nu^2\right)
		\end{align*}
		As a result, the expectation of the log transformation has 
		\begin{align}\label{equ:lowerbound1}
			\mathbb{E}\left[\sum_{j=1}^m\log \left(1 - \frac{y_{i,j} X_{i,j}^k z_k }{1+y_{i,j}\left\langle X_{i,j}, \beta_z\right\rangle} \cdot \frac{4\sqrt{2  } \nu}{{ s^*}}\right) \right] \leq m \cdot \log \left(1 +64\nu^2\right)
		\end{align}
		Moreover, since $1 - 5 |x| \leq \log(1 + x) \leq 1+ 5 |x|$ for $|x| \leq 1 / 2$, we have
		\begin{align*}
			1 - \frac{10\sqrt{2  } \nu}{\sqrt{ s^*}}\leq \log \left(1 - \frac{y_{i,j} X_{i,j}^k z_k }{1+y_{i,j}\left\langle X_{i,j}, \beta_z\right\rangle} \cdot \frac{4\sqrt{2  } \nu}{\sqrt{ s^*}} \right)\leq1 +  \frac{10\sqrt{2  } \nu}{\sqrt{ s^*}}.
		\end{align*}
		Recall that $|\{z\in \{-1,1\}^d| \sum_j\eins\{z^j = 1\}  \leq s^{*}\}| \leq d^{s^{*}}$. 
		Thus, applying Hoeffding's inequality with union bound yields
		\begin{align}\nonumber
			& \left|\sum_{j=1}^m\log \left(1 - \frac{y_{i,j} X_{i,j}^k z_k }{1+y_{i,j}\left\langle X_{i,j}, \beta_z\right\rangle} \cdot \frac{4\sqrt{2  } \nu}{\sqrt{ s^*}}\right) - \mathbb{E}\left[\sum_{j=1}^m\log \left(1 - \frac{y_{i,j} X_{i,j}^k z_k }{1+y_{i,j}\left\langle X_{i,j}, \beta_z\right\rangle} \cdot \frac{4\sqrt{2  } \nu}{\sqrt{ s^*}}\right) \right]\right|\\
			\leq & \frac{20\nu \sqrt{m (\log n + s^*\log d)} }{\sqrt{s^*}} \leq \frac{20\nu \sqrt{md }}{{s^*}}
			\label{equ:lowerbound2}
		\end{align}
		for all $1\leq i \leq n$ and $z\in \{-1,1\}^d $ with $\sum_j\eins\{z^j = 1\}  \leq s^{*}$ with probability at least $1 - 2 / n^2$. 
		As a result, plugging \eqref{equ:lowerbound1} and \eqref{equ:lowerbound2} into \eqref{equ:lowerboundratioofP} yields 
		\begin{align*}
			\frac{d \mathrm{P}_{z^{\oplus k}}}{d \mathrm{P}_z} = & \exp\left(\sum_{j=1}^m\log \left(1 - \frac{y_{i,j} X_{i,j}^k z_k }{1+y_{i,j}\left\langle X_{i,j}, \beta_z\right\rangle} \cdot \frac{4\sqrt{2  } \nu}{\sqrt{ s^*}}\right) \right)\\
			\leq & \exp\left( m \cdot \log \left(1 + 64\nu^2\right) +\frac{20\nu \sqrt{m d} }{{s^*}} \right).
		\end{align*}
		Since $n\varepsilon^2 \geq s^{*2}$, for sufficiently small $\nu$, one can justify condition \ref{asp:densityexist} and \ref{asp:orthogonality} for the defined $\mathrm{P}_z$, with $	\alpha^2 \asymp   \frac{\nu^2{m d} }{{s^{*2}}} $. 
		Applying Corollary 1 in \citet{acharya2020unified} leads to 
		\begin{align}\label{equ:ourlowerbound1}
			\left(\frac{1}{d} \sum_{i=1}^d \mathrm{~d}_{\mathrm{TV}}\left(\mathrm{P}_{+i}^{S^n}, \mathrm{P}_{-i}^{S^n}\right)\right)^2 \lesssim \frac{nm \nu^2\varepsilon^2 }{s^{*2}}. 
		\end{align}
		Note that this result, as well as Lemma 3 of \citet{acharya2020unified} in the following, are developed for $X_i$ being a single sample.
		They are extendable to $X_i$ being multiple samples since we can apply the original conclusion to the $m(d+1)$ dimensional vector, formulated by stacking the $X_{i,j}$s. 
		Next we focus on lower bound of the total variation distance. 
		Since 
		\begin{align*}
			\left\|\beta_z-\beta_{z^{\prime}}\right\|_2=\sqrt{\frac{32 \nu^2}{s^*} \sum_{i=1}^d \eins\left\{Z_i \neq \hat{Z}_i\right\}}=4 \nu\left(\frac{d_{\operatorname{Ham}(z, \hat{z})}}{\tau d}\right)^{1 / 2},
		\end{align*}
		i.e. Condition \ref{asp:addtive} holds, applying Lemma 3 of \citet{acharya2020unified} leads to 
		\begin{align}\label{equ:ourlowerbound2}
			\frac{1}{d} \sum_{i=1}^d \mathrm{~d}_{\mathrm{TV}}\left(\mathrm{P}_{+i}^{S^n}, \mathrm{P}_{-i}^{S^n}\right) \geq \frac{1}{4}.
		\end{align}
		Combining \eqref{equ:ourlowerbound1} and \eqref{equ:ourlowerbound2} leads to the desired conclusion.

	\end{proof}

	\begin{proof}[\textbf{Proof of Proposition \ref{lem:necessarym}}]
		We follow the same construction as in the proof of Theorem \ref{thm:ourlowerbound} while adopting a different strategy to bound $\frac{d \mathrm{P}_{z^{\oplus k}}}{d \mathrm{P}_z} $.
		Namely, we let 
		\begin{align}\label{equ:lowerboundratioofPmlarge}
			\frac{d \mathrm{P}_{z^{\oplus k}}}{d \mathrm{P}_z}  =
			\prod_{j=1}^m 1 - \frac{y_{i,j} X_{i,j}^k z_k }{1+y_{i,j}\left\langle X_{i,j}, \beta_z\right\rangle} \cdot \frac{4\sqrt{2  } \nu}{\sqrt{ s^*}} \leq \left( 1 +  \frac{8\sqrt{2  } \nu}{\sqrt{ s^*}} \right)^m.
		\end{align}
		Then one can justify condition \ref{asp:densityexist} and \ref{asp:orthogonality} for the defined $\mathrm{P}_z$, with
		\begin{align*}
			\alpha^2 \asymp \left(\left( 1 +  \frac{8\sqrt{2  } \nu}{\sqrt{ s^*}} \right)^{m} -1\right)^2. 
		\end{align*}
		Applying Corollary 1 in \citet{acharya2020unified} leads to 
		\begin{align}\label{equ:ourlowerbound1mlarge}
			\left(\frac{1}{d} \sum_{i=1}^d \mathrm{~d}_{\mathrm{TV}}\left(\mathrm{P}_{+i}^{S^n}, \mathrm{P}_{-i}^{S^n}\right)\right)^2 \lesssim \frac{n\varepsilon^2 }{d} \cdot \left( \left( 1 +  \frac{8\sqrt{2  } \nu}{\sqrt{ s^*}} \right)^{m} -1\right)^2. 
		\end{align}
		There holds similarly 
		\begin{align}\label{equ:ourlowerbound2mlarge}
			\frac{1}{d} \sum_{i=1}^d \mathrm{~d}_{\mathrm{TV}}\left(\mathrm{P}_{+i}^{S^n}, \mathrm{P}_{-i}^{S^n}\right) \geq \frac{1}{4}.
		\end{align}
		Combining \eqref{equ:ourlowerbound1mlarge} and \eqref{equ:ourlowerbound2mlarge} leads to 
		\begin{align*}
			\exp\left(\frac{\nu m}{\sqrt{s^*}} \right)\asymp   \left( 1 +  \frac{8\sqrt{2  } \nu}{\sqrt{ s^*}} \right)^{m} \gtrsim 1 + \sqrt{\frac{d}{n\varepsilon^2}}.
		\end{align*}
		which yields 
		\begin{align*}
			\nu^2 \gtrsim \frac{{s^*}}{m^2 } \log^2 \left(1 + \sqrt{\frac{d}{n\varepsilon^2}}\right) . 
		\end{align*}
		Note that if $n\varepsilon^2 \lesssim  \sqrt{d}$ and $m\leq \log d$, there holds
		\begin{align*}
			\nu^2  \gtrsim \frac{{s^*}}{m^2}  \log^2\left( 1 + \sqrt{\frac{d}{n\varepsilon^2}}\right) \gtrsim \frac{{s^*}}{m^2}  \log^2\left( 1 + d^{1/4}\right) \gtrsim \frac{ \log^2 d}{s^*\log^2 d} = \frac{1}{s^*}
		\end{align*}
		which yields the desired result. 
		Note that in this case, the constructed function class has beta-min condition with $a = {\nu  / \sqrt{s^*}} \gtrsim 1$ which is a constant in a $[0,1]$. 
		
	\end{proof}

	\section{Candidate Variable Selection}\label{app:candidatevariableselection}

	\subsection{Good Selectors}\label{app:pluginselector}

	\subsubsection{Plug-in High Dimensional Variable Selection}
	
	In the following, we provide some example selectors and demonstrate that, under mild assumptions, they serve as components of a good selector. 
	We introduce commonly used variable selection approaches along with their associated theoretical results.
	Our goal is twofold. Firstly, we want the true variables to be selected.
	Conversely, the redundant variables that are selected should be as few as possible. 
	We derive this from the perfect selection property (also known as strong oracle or consistent selection), which asserts that our goal is achieved with a high probability.
	The primary conditions we impose on the potential distributions fall into two categories:
	\begin{itemize}
		\item Beta-min conditions, which necessitate that $\min_{\beta^{*j}>0}| \beta^{*j}|$ is greater than a specified threshold. 
		With this condition, the signal strength from the regression functions is robust enough for the selector to identify the variables..
		\item Mild correlation conditions, which require that the correlation between the true and redundant variables is weak enough for the selectors to distinguish.
	\end{itemize}
	In this section, we omit the user index $i$ and write $(X, y)$ representing the data of some user $(X_i, y_i)$, since the results in the section consider one local dataset at a time.

	\begin{example}[\textbf{Lasso \citep{tibshirani1996regression}}]\label{example:lasso}

		Lasso, or Least Absolute Shrinkage and Selection Operator, is a regularization technique in statistical learning that adds a penalty term to the linear regression objective function, effectively promoting sparsity by encouraging some of the model coefficients to be exactly zero. 
		Specifically, Lasso solves the regularized optimization object 
		\begin{align}\label{equ:lassoobject}
			\min _{\beta \in \mathbb{R}^d}\left\{\frac{1}{n}\|y-X \beta\|_2^2+\lambda\|\beta\|_1\right\}.
		\end{align}
		Used for variable selection, Lasso identifies the non-zero elements of the optimization solution as the selected variable. 
	\end{example}

	To study the selection consistency of Lasso, \citet{zhao2006model} proposed a general condition called the Irrepresentable condition.
	Specifically, for $\widehat{\Sigma} = {X}^{\top} {X} / n $, let the block matrix 
	\begin{align*}
		\widehat{\Sigma}=\left(\begin{array}{ll}
			\widehat{\Sigma}_{11} & \widehat{\Sigma}_{12}\\
			\widehat{\Sigma}_{21} & \widehat{\Sigma}_{22}
		\end{array}\right) .
	\end{align*}
	Here $\widehat{\Sigma}_{11}$ is a $s^*\times s^*$ matrix, corresponding to the covariance matrix of the true variables. 
	Irrepresentable Condition states that there exists a positive constant vector $\eta$
	\begin{align}\label{equ:irrepresentablecondition}
		\left|	\widehat{\Sigma}_{21}\left(	\widehat{\Sigma}_{11}\right)^{-1} \operatorname{sign}\left(\beta^{*1:s^*}\right)\right| \leq \mathbf{1}-\eta,
	\end{align}
	where $\mathbf{1}-\eta$ is a $d - s^*$ vector with $1 - \eta$ elementwisely. 
	The following result holds for irrepresentable condition.
	\begin{lemma}\label{lem:oracleforlasso}
		Under our assumptions, when using \eqref{equ:lassoobject} as selector, let $\beta_{LASSO}$ be the solution. 
		Suppose \eqref{equ:irrepresentablecondition} holds. 
		Suppose the following conditions hold: 
		(\textit{i}) $m\gtrsim s^{*2}  \log d$. 
		(\textit{ii}) $\min_{\beta^{*j}>0} |\beta^{*j}| \gtrsim \sqrt{{1}/{m}}$. 
		Then there exists a constant $C_p < 1$ such that, for sufficiently large $m$, with probability $C_p$, there holds
		\begin{align*}
			\beta_{LASSO}^{j} \neq 0 \;\; \text{ for }\;\; j = 1,\cdots, s^* \quad \text{ and } \; \beta_{LASSO}^{j} = 0 \text{ for } j = s^* + 1,\cdots, d. 
		\end{align*}
		Moreover, for $\Sigma = \mathbb{E}XX^{\top}$, if $|\Sigma_{ij}| \leq 3/ s^*$ for $i\neq j$, then we have the Irrepresentable Condition. 
	\end{lemma}

	\begin{proof}[\textbf{Proof of Lemma \ref{lem:oracleforlasso}}]
		Since we assume sub-Gaussian noises, any $k$-th moment of the random noise exists, i.e. $k$ can be arbitrarily large. 
		As a result, any $\lambda \gtrsim \sqrt{m}$ implies $(\lambda / \sqrt{m})^2k / d \to \infty$ for some $k$. 
		By Theorem 3 in \citet{zhao2006model}, for sufficiently large $m$, the probability of 
		\begin{align*}
			\mathrm{sign}(\beta_{LASSO}^{j}) = \mathrm{sign}(\beta^{*j})\quad  \text{ for } j =  1,\cdots, d
		\end{align*}
		is larger than some constant $C_p$, given that the conditions (5,6,7,8) are satisfied. 
		Thus it suffices to verify the conditions. 
		Condition (5) and (6) holds naturally due to our assumption of i.i.d. designs and boundedness of covariance matrix norm. 
		(7) and (8) are in our assumptions. 
		As for the last statement, \citet{zhao2006model} provides several commonly seen sufficient conditions for the irrepresentable condition to hold, such as when $|\widehat{\Sigma}_{ij}| \leq 1 / (2s^* - 1)$. 
		If $|\Sigma_{ij}| \leq 1 / 3 s^*$, then $|\widehat{\Sigma}_{ij}| \leq |\Sigma_{ij}| + |\Sigma_{ij} - \widehat{\Sigma}_{ij}| \leq 1 / 3s^* + c / \sqrt{m} \leq  1 / (2s^* - 1) $ for some constant $c$ and sufficiently large $m \gtrsim s^*$. 
		This bound holds for all users and all position $i$, $j$ if we apply union bound, where we need $ \log d / m \lesssim 1 / s^{*2}$, i.e. $m \gtrsim s^{*2}  \log d$. 
		Note here we assumed $d \gtrsim n$. 
		Thus the lemma is proved.
	\end{proof}

	\begin{example}[\textbf{SCAD \cite{fan2001variable}}]\label{example:scad}
		SCAD, or smoothly clipped absolute deviation, is a non-convex penalty function used in statistical learning and regression analysis. It is designed to address limitations of traditional L1 regularization methods like Lasso by providing a smooth and more robust penalty on regression coefficients, promoting sparsity while mitigating some of the biases associated with sharp discontinuities in penalty functions.
		Specifically, SCAD solves the regularized optimization object 
		\begin{align}\label{equ:scadobject}
			\min _{\beta \in \mathbb{R}^d}\left\{\frac{1}{n}\|y-X \beta\|_2^2+\lambda \sum_{j=1}^d \psi_\lambda\left(\beta_j\right)\right\} \text{ where } \psi_\lambda^{\prime}(t)=\lambda I_{\{t \leq \lambda\}}+\frac{(a \lambda-t)_{+}}{a-1} I_{\{t>\lambda\}}  \;\; \text { for some } a>2. 
		\end{align}
		Used for variable selection, SCAD identifies the non-zero elements of the optimization solution as the selected variable. 
	\end{example}

	The following lemma, which is a straightforward implication of  \citet{fan2011nonconcave}, states that the essential condition for SCAD estimator to consistently select the variables is the Beta-min condition, given that the sample size is relatively large.
	\begin{lemma}\label{lem:oracleforscad}
		Under our assumptions, when using \eqref{equ:scadobject} as selector, let $\beta_{SCAD}$ be the solution. 
		Suppose the following conditions hold: (\textit{i}) The sparsity $s^*$ is $\mathcal{O}(1)$. 
		(\textit{ii}) $m\gtrsim s^* \vee \log m \log d$. 
		(\textit{iii}) $\min_{\beta^{*j}>0}| \beta^{*j} |\gtrsim \sqrt{{s^*}/{m}}\vee \sqrt{{\log d \log m}/{m}}$. 
		Then there exists a constant $C_p < 1$ and a suitable choice of $\lambda_m$ such that, for sufficiently large $m$, with probability $C_p$, there holds
		\begin{align*}
			\beta_{SCAD}^{j} \neq 0 \;\; \text{ for }\;\; j = 1,\cdots, s^* \quad \text{ and } \; \beta_{SCAD}^{j} = 0 \text{ for } j = s^* + 1,\cdots, d. 
		\end{align*}
	\end{lemma}
	
	\begin{proof}[\textbf{Proof of Lemma \ref{lem:oracleforscad}}]
		By Theorem 3 in \citet{fan2011nonconcave}, for sufficiently large $m$, the probability of 
		\begin{align*}
			\|\beta_{SCAD} - \beta^*\|_2 \lesssim \sqrt{\frac{s^*}{m}} \;\; \text{ for }\;\; j = 1,\cdots, s^* \quad \text{ and } \; \beta_{SCAD}^{j} = 0 \text{ for } j = s^* + 1,\cdots, d
		\end{align*}
		is larger than some constant $C_p$, given that the regularity conditions in the theorem are satisfied. 
		Note that we have $|\beta_{SCAD}^{j}|\geq | \beta^{*j} |- \sqrt{s^* / m} \geq  | \beta^{*j} | / 2 > 0$. 
		Thus it suffices to verify the conditions. 
		The condition 1 is satisfied by SCAD penalty.
		Condition 5 is satisfied by our setting of sample size (note that $\log d \lesssim n^{\alpha'}$ for $\alpha'$ defined in their context). 
		(26) and (28) of Condition 2 follows from our assumptions on the upper and lower bound of $\|\mathbb{E}XX^{\top}\|_2$ and the estimation error of covariance matrix which is $\mathcal{O}(\sqrt{s^* / m})$ \cite{wainwright2019high}. 
		(27) comes from $s^* = \mathcal{O}(1)$. 
	\end{proof}

	\subsubsection{Proof of Proposition \ref{prop:existenceofgoodselectors}}

	\begin{proof}[\textbf{Proof of Proposition \ref{prop:existenceofgoodselectors}}]
		Under the two conditions, using Lemma \ref{lem:oracleforlasso} and \ref{lem:oracleforscad}, we can show that there exists a variable selection method that perfectly select the true variables with a positive probability $C_p$. 
		Then by sampling among the selected variables, the probability can be computed as 
		\begin{align*}
			\mathrm{Pr} \left(  \mathcal{S}(X_i, y_i) = v \right) \geq C_p \mathrm{Pr} \left(v = j \; \text{ for }\; v \sim \text{Unif}\left(1,\cdots, s^*\right)\right)\geq \frac{C_p}{s^*}
		\end{align*}
		for $1\leq v \leq s^*$. 
		This yields the desired conclusion. 
	\end{proof}

	\subsubsection{Computational Issue}
	
	For SCAD, the incorporation of a non-convex penalty proves effective in attaining coefficient sparsity while maintaining oracle properties. 
	Nonetheless, the non-convex nature introduces a challenge—the guarantee of solution uniqueness becomes elusive, leading to the presence of multiple local optima. 
	Consequently, the stability of results may be compromised. 
	\citet{fan2014strong} introduce additional concave parameter to ensure consistency, which contributes to increased computational complexity, further posing challenges in the computational efficiency of SCAD. 
	As a result, Lasso is more preferable In practice. 
	We introduce another technique which can be useful to enhance the computation efficiency.

	\begin{example}[\textbf{Screening \cite{fan2008sure}}]
		Sure Independence Screening (SIS) is a feature selection method in statistical learning that aims to identify relevant variables in high-dimensional datasets. It does so by assessing the correlation between each predictor and the response variable, and selecting a subset with the highest scores.
		Specifically, 
		For $(X,  y) \in ( \mathcal{X} \times \mathcal{Y})^m$, let 
		\begin{align*}
			w = X^{\top} y. 
		\end{align*}
		Then the $s$ most largest position of $w$ are identified as the selected variables.	
		Screening can be a valuable pre-procedure for other selection methods. 
		Screening is employed to quickly identify and retain a subset of potentially important features, reducing the dimensionality of the data before applying more computationally intensive or elaborate feature selection techniques.
	\end{example}

	\subsection{Aggregation of Local Selected Variables}\label{app:heavyhitter}

	In this section, we present the omitted algorithm and technical proofs for the aggregation step after local variable selection. 
	In \ref{app:heavyhitteralgorithmintro}, we introduce the detailed variable selection algorithm. 
	In \ref{app:proofrelatedtoheavyhitter}, we present proofs omitted in Section \ref{sec:candidatevariableselection}.

	\subsubsection{Heavy Hitter Algorithm}\label{app:heavyhitteralgorithmintro}
	
	First, we introduce necessary definitions. 
	Let $\mathcal{V}$ be a collection of binary prefixes. 
	The define \texttt{ChildSet} $= \{v +0 , v+ 1 \text{ for } v \in \mathcal{V}\}$. 
	We define several public randomness that will be shared among users. 
	See \citet[Section 3.1]{bassily2020practical} for details. 
	Let $\overline{\mathcal{V}}=\left\{v \in\{0,1\}^{\ell}\right.$ for some $\ell \in$ $[\log d]\}$.
	Define integer $t= 3\log (n )$ and $k=O(\sqrt{{n}/{3\log (n )}})$.
	We will consider a set of $t$ pairs of hash functions $\left\{\left(h_1, g_1\right), \ldots,\left(h_t, g_t\right)\right\}$, where for each $i \in[t], h_i: \overline{\mathcal{V}} \rightarrow[k]$ and $g_i: \overline{\mathcal{V}} \rightarrow\{-1,+1\}$ are independently and uniformly chosen pairwise independent hash functions.
	We assume that the server creates a random partition $\Pi:[n] \rightarrow[\log d] \times[k]$ that assigns to each user $i \in[n]$ a random pair $\left(\ell_i, j_i\right) \leftarrow[\log (d)] \times[k]$, as in the initialization of Algorithm \ref{alg:heavyhitter}.
	We also have another random function $\mathcal{Q}:[n] \leftarrow[k]$ that assigns to each user $i$ a uniformly random index $r_i \leftarrow[k]$. We assume that such random indices $\ell_i, j_i, r_i$ are shared between the server and each user.
	Finally, we adopt shared encoding and decoding schemes for bijection between $[d]$ and $\lceil\log d\rceil$ binary strings, denoted as \texttt{Encoding} and \texttt{Decoding}, respectively.

	Before presenting the \texttt{HeavyHitter}, we first introduce the functions it uses. 
	The following algorithm generate a private report for a single user. 
	We seal the information of each $v_i$ into a binary value that is the Hardamard transform of hashes of its prefix. 
	The information is privatized using the random response mechanism \citep{warner1965randomized} and sent to the curator.

	\begin{algorithm}[htbp]
		\caption{\texttt{LocalRnd} \citep{bassily2020practical}}
		\label{alg:localrnd}
		\begin{algorithmic}
			\STATE {\bfseries Input: }{ Privacy budget $\varepsilon$, input $v_i$. } \\
			\STATE Compute $\tilde{v}_i = \texttt{Encoding}(v_i)$ the binary string encoding.  \\
			\STATE Using pubic randomness to get $(\ell_i, j_i)$ and $r_i$. \\
			\STATE Let $s_i:=g_{j_i}\left(\tilde{v}_i\left[1: \ell_i\right]\right) $ and $c_i:=h_{j_i}\left(\tilde{v}_i\left[1: \ell_i\right]\right)$. Here $ v[1: \ell]$ denote the $ \ell $ -bit prefix of  $v$. \\
			\STATE Compute $x_i = s_i \cdot W_{r_i, c_i}$. Here $W_{r,c}$ denotes the sign of $(r, c)$ entry of Hadamard matrix with size $k$.\\
			\STATE Random permute $x_i$ with \begin{align*}
				y_i=\left\{\begin{array}{cc}
					x_i & \text { w.p. } \frac{e^{\epsilon }}{e^{\epsilon }+1} \\
					-x_i & \text { w.p. } \frac{1}{e^{\epsilon }+1}
				\end{array}\right.
			\end{align*}
			\STATE {\bfseries Output: }{$y_i$.}
		\end{algorithmic}
	\end{algorithm}

	The following algorithm shows how \texttt{LocalRnd} is invoked multiple times to scan the prefix tree.

	\begin{algorithm}[htbp]
		\caption{\texttt{FreqOracle} \citep{bassily2020practical}}
		\label{alg:freqoracle}
		\begin{algorithmic}
			\STATE {\bfseries Input: }{ Prefixes length $\ell$, a subset of $\ell$-bit prefixes $\widehat{\mathcal{V}} \subseteq\{0,1\}^{\ell}$, collection of $t$ disjoint subsets of users: $\left\{\tilde{\mathcal{I}}_j: j \in[t]\right\}$, privacy budget $\varepsilon$. } \\
			\FOR{$\widehat{v}\in\widehat{\mathcal{V}} $}
			\FOR{Hash index $j= 1$ to $t$ }
			\STATE Let $s:=g_j(\widehat{v}) $ and $ c:=h_j(\widehat{v})$. 
			\FOR{$i \in \tilde{\mathcal{I}}_j$}
			\STATE $y_i$ = \texttt{LocalRnd}($\varepsilon$, $v_i$).
			
			\ENDFOR
			\STATE Compute the j-th estimate of the frequency of $\widehat{v}$:  $\widehat{f}_j(\widehat{v})= t \log d \cdot \frac{e^{\varepsilon} + 1}{e^{\varepsilon} - 1 } \sum_{i \in \tilde{\mathcal{I}}_j} y_i \cdot s \cdot W_{r_i, c}$.
			
			\ENDFOR
			\STATE The final estimation of $\widehat{v}: \widehat{f}(\widehat{v}):=\operatorname{Median}\left(\left\{\widehat{f}_j(\widehat{v}): j \in[t]\right\}\right)$. 
			
			\ENDFOR
			\STATE {FreqList} $ =\{(\widehat{v}, \widehat{f}(\widehat{v})): \widehat{v} \in \widehat{\mathcal{V}}\} .$
			
			{\bfseries Output: }{FreqList}
		\end{algorithmic}
	\end{algorithm}

	The final algorithm is presented in Algorithm \ref{alg:heavyhitter}. 
	We modify the algorithm in \citet{bassily2020practical} by removing the second phase of frequency estimation, since we only want to identify the heavy hitters and do not care about their frequencies. 
	This allows a saving of $\varepsilon / 2$ budget.

	\begin{algorithm}[htbp]
		\caption{\texttt{HeavyHitter}}
		\label{alg:heavyhitter}
		\begin{algorithmic}
			\STATE	{\bfseries Input: }{ User values $\mathcal{V} = \{v_i \in [d]\}$, privacy budget $\varepsilon$, threshold $\rho$. }\\
			\STATE	{\bfseries Initialization: }{ Prefixes $ = \{\}$, public randomness pairs $\Gamma = \{(\ell_i, j_i) \in [\log d]\times [3 \log n]  \text{ for } 1\leq i \leq n \}$, partition $I_{\ell, j} = \{i \text{ if }(\ell_i, j_i) = (\ell, j) \}$. }\\
			\FOR{$\ell$ in $1, \cdots, \lceil\log d \rceil$}
			\STATE $ \{(\widehat{v}, \widehat{f}(\widehat{v})): \widehat{v} \in\texttt{ChildSet} (\text{Prefixes}) \}=  \texttt{ FreqOracle}\left(\ell, \text { ChildSet (Prefixes) },\left\{\mathcal{I}_{\ell, j}: j \in[3 \log n]\right\}, \varepsilon\right)$.\\
			\STATE Let NewPrefixes $ = \{\}$. 
			\FOR{$v\in\texttt{ChildSet}(\text{Prefixes})$}
			\IF{$\widehat{f}(\widehat{v}) \geq\rho n$}
			\STATE Add $\widehat{v}$ to NewPrefixes. 
			\ENDIF
			\ENDFOR
			\IF{$|\text{NewPrefixes}| = 0$}
			\STATE Add $\arg\max_{\widehat{v}} \widehat{f}(\widehat{v})$ to NewPrefixes. \texttt{\# Ensure NewPrefixes is non-empty.}
			\ENDIF
			\STATE Prefixes $\leftarrow$ NewPrefixes. 
			\ENDFOR
			\STATE {\bfseries Output: }{ $\{\texttt{Decoding}(v) \text{ for } (v, \widehat{f}(v)) \in \text{ Prefixes}  \}$}.
		\end{algorithmic}
	\end{algorithm}

	\subsubsection{Proof Related to Section \ref{sec:candidatevariableselection}}\label{app:proofrelatedtoheavyhitter}

	To give the proof of Proposition \ref{prop:heavyhitterselection}, we need the following necessary technical result. 
	
	\begin{lemma}\label{lem:bassilyheavyhitter}
		Algorithm \ref{alg:heavyhitter} is $\varepsilon$- ULDP. 
		Moreover, if $\alpha \gtrsim s^* \sqrt{\log n \log d / n} / \varepsilon$, then with probability at least $1-1 / n^2$, the output list of the \texttt{HeavyHitter} protocol satisfies the following properties given sufficiently large $n$:
		(\textit{i}) it contains all items $v \in \mathcal{V}$ whose true frequencies above $2 \rho n $.
		(\textit{ii}) it does not contain any item $v \in \mathcal{V}$ whose true frequency below $\rho n  / 2$.
	\end{lemma}

	\begin{proof}[\textbf{Proof of Lemma \ref{lem:bassilyheavyhitter}}]
		
		Lemma 5.3 in \citet{bassily2020practical} yields that the variables $v$ retained in Prefixes in Algorithm \ref{alg:heavyhitter} has $|\widehat{f}(v) - f(v)| \lesssim {\sqrt{n \log n \log d}}/{\varepsilon}$. 
		Since $\alpha \gtrsim s^* \sqrt{\log n \log d / n} / \varepsilon$, we have $ {\sqrt{n \log n \log d}}/ {\varepsilon} \leq \rho n / 2$ for sufficiently large $n$. 
		Then for any $v$ in Prefixes, we have $	f(v) \gtrsim \rho n   -  {\sqrt{n \log n \log d}}/ {\varepsilon} \geq \rho n /2$. 
		On the contrary, if $f(v) \geq 2 \rho n $, then $\widehat{f}(v) \gtrsim 2 \rho n  -  {\sqrt{n \log n \log d}}/{\varepsilon}\geq \rho n$, which will be included in Prefixes.

	\end{proof}


	\begin{proof}[\textbf{Proof of Proposition \ref{prop:heavyhitterselection}}]
		
		For notation simplicity, we denote the number of users and selectors used in the selection as $n$ instead of $n/ 2$ throughout this proof.
		We compute the frequency of variable $j$, namely $\sum_{i=1}^{n} \eins\left(v_i = j\right)$.
		By Hoeffding's inequality, we have 
		\begin{align*}
			\mathrm{Pr}\left(	\left|\sum_{i=1}^n \eins\left(v_i = j\right) -  \sum_{i=1}^n \mathrm{Pr}\left(v_i = j\right)\right| \geq \sqrt{n (\log nd) }\right) \leq 2 \exp \left( -2 (\log n +  \log d)  \right). 
		\end{align*}
		Applying union bound, we get 
		\begin{align}\nonumber
			\mathrm{Pr}\left(	\left|\sum_{i=1}^n \eins\left(v_i = j\right) -  \sum_{i=1}^n \mathrm{Pr}\left(v_i = j\right)\right| \geq \sqrt{n  \log nd }\quad \text{ for  } 1\leq j \leq d \right) \leq & 2d \exp \left( -2 (\log n +  \log d) \right)\\
			< & \exp\left(-2\log n \right) = 1 / n^2. \label{equ:unionboundofdiscretedistribution}
		\end{align}
		For conclusion (\textit{i}), since $v_i$ is generated by a good selector, Definition \ref{def:successfulscreener} yields that
		\begin{align*}
			\sum_{i=1}^n \mathrm{Pr}\left(v_i = j\right)  \geq \frac{n \alpha}{s^*}
		\end{align*}
		for $j = 1, \cdots, s^*$. 
		This together with \eqref{equ:unionboundofdiscretedistribution} leads to 
		\begin{align*}
			\sum_{i=1}^n \eins\left(v_i = j\right) \geq\frac{n \alpha}{s^*}  - \sqrt{n \log nd} \geq\frac{n \alpha}{2s^*}
		\end{align*}
		for any $1\leq j \leq s^*$ and sufficiently large $n$. 
		Then for any $\rho \leq \alpha / 4 s^*$,  by Lemma \ref{lem:bassilyheavyhitter}, we have $\widehat{f}(j) \geq \rho n$.
		This means the frequency of any true variable must be large enough to be detected as a heavy hitter.
		Next, we show (\textit{ii}).
		Suppose that there are $s$ variables $j_1,\cdots, j_s$ satisfying $\sum_{i=1}^n \eins\left(v_i = j\right)  \geq \rho n / 2$, i.e. potentially identified by the heavy hitters by Lemma \ref{lem:bassilyheavyhitter}. 
		Then by applying \eqref{equ:unionboundofdiscretedistribution}, there holds
		\begin{align*}
			\sum_{k = 1}^s\sum_{i=1}^n \mathrm{Pr}\left(v_i = j_k\right)  \geq s \cdot \frac{\rho n}{2}  - s \sqrt{n  \log nd } \geq  s \cdot \frac{\rho n }{4}
		\end{align*}
		for sufficiently large $n$, with probability at least $1 - 1 / n^2$. 
		However, there holds 
		\begin{align*}
			s\cdot  \frac{\rho n }{4} \leq \sum_{k = 1}^s\sum_{i=1}^n \mathrm{Pr}\left(v_i = j_k\right) \leq \sum_{j = 1}^d\sum_{i=1}^n \mathrm{Pr}\left(v_i = j\right)  = n, 
		\end{align*}
		which indicates that $s \leq 4 / \rho  \leq 32 s^* / \alpha$. 
		
	\end{proof}

	\section{Coefficient Estimation}\label{app:coeffestimaton}

	\subsection{The Multiple Round Protocol}\label{app:multiroundprotocol}

	\subsubsection{SCO Algorithm}\label{app:scoalgorithm}
	
	We use the same algorithm as in the \citet{bassily2023user} while adopting a different set of default values of its parameters.
	Such changes are due to the differential technical requirements for the theoretical analysis with strong convexity. 
	Also, the algorithm requires a solution to the user-level locally differentially private mean estimation (\texttt{ULDPMean}), which is presented later in Section \ref{app:uldpmeanestimation}. 
	For notation simplicity, we denote the number of users and selectors used in the selection as $n$ instead of $n/ 2$ in this section.

	\begin{algorithm}[H]
		\caption{\texttt{ULDPSCO} }
		\label{alg:sco}
		\begin{algorithmic}
			\STATE {\bfseries Input: }{ Local data sets $\{(X_i, y_i)\}_{i=1}^n$, number of iterations $T$, concentration radius $\tau$,  privacy budget $\varepsilon$. } 
			\STATE {\bfseries Initialization : }{$\beta_0=\overrightarrow{0}$, $\beta^{a g}=\beta_0$, and $\left\{\eta_t, \gamma_t\right\}_{t \in[T]}$ as in Lemma \ref{lem:convergenceofsco}. }
			\FOR{$t = 0, 1,\cdots, T-1$}
			\STATE Compute $\beta_t^{m d}=\gamma_t^{-1} \beta_t+\left(1-\gamma_t^{-1}\right) \beta_t^{a g}$.
			\STATE Choose two fresh batches $S_{t,1}$ and $S_{t,2}$ of $n_0=\lfloor n / 2 T\rfloor$ users, respectively.
			\STATE Compute the average gradient at each user at $\beta_t$,
			$g_i\left(\beta_t^{m d}\right)=\frac{1}{m L} \sum_{j=1}^m\left( X_{i, j}^{\top}\beta_t^{m d} -  y_{i,j} \right)X_{i, j}$ for $i \in S_{t,1} \cup S_{t,2}$.
			\STATE Compute the average gradients
			$\tilde{\nabla} F\left(\beta_t^{m d}\right)=\texttt{ULDPMean}(\{{g_i\left(\beta_t^{m d}\right)}_i\}_{i \in S_{t,1}}, \{{g_i\left(\beta_t^{m d}\right)}_i\}_{i \in S_{t,2}}, \tau, \varepsilon). $
			\STATE Update $\beta_{t+1}=\beta_t^{m d}-\eta_t \cdot L \cdot \tilde{\nabla} F\left(\beta_t^{m d}\right)$.
			\STATE Compute $\beta_{t+1}^{a g}=\gamma_t^{-1} \beta_{t+1}+\left(1-\gamma_t^{-1}\right) \beta_t^{a g}$.
			\ENDFOR
			\STATE {\bfseries Output: }{$\beta_T^{a g}$. }
		\end{algorithmic}
	\end{algorithm}

	For the algorithm, we have the following result. 
	Note that Algorithm \ref{alg:sco} adopts disjoint mini-batch when computing the gradients while Lemma \ref{lem:noisygradient} was established completely based on stochastic gradient descent. 
	Yet, the theoretical analysis generalize straightforwardly. 
	
	\begin{lemma}[\textbf{Theorem 3 of \citet{dieuleveut2017harder}}]\label{lem:noisygradient}
		Consider the stochastic convex optimization problem \eqref{equ:scoobject}. 
		Suppose each $\tilde{\nabla} F(\beta)$ is an unbiased stochastic oracle to $\nabla F(\beta)$ with variance $\nu^2$. 
		Let $\beta_T^{\prime ag }$ be the associated non-private output of Algorithm \ref{alg:sco} ($\varepsilon = \infty$). 
		There exists settings of $\left\{\eta_t, \gamma_t\right\}_{t \in[T]}$ such that
		\begin{align*}
			\mathbb{E}\left[F\left(\beta_T^{\prime ag }\right)-\min _{\beta } F(\beta) \right]\lesssim \frac{s \nu^2}{T} + \frac{\|\beta^*\|_2^2\lambda_n\left(\mathbb{E}\left[XX^{\top}\right]\right)^{-1}}{T^2} .
		\end{align*}
	\end{lemma}

	For clearness, we additionally include the full multi-round protocol. 
	
	\begin{algorithm}[htbp]
		\caption{Multi-round ULDP sparse linear regression.}
		\label{alg:uldpsparselinearregressionmultiround}
		\begin{algorithmic}
			\STATE	{\bfseries Input: }{ Local data sets $\{(X_i, y_i)\}_{i=1}^n$, selectors $\{\mathcal{S}_i\}_{i=1}^{n/2}$, privacy budget $\varepsilon$,  threshold $\rho$.}
			\STATE {\bfseries Initialization: } ${\beta} \in \mathbb{R}^d$ be a zero vector. 
			\STATE {\color{red}\texttt{\# candidate variable selection}}
			\STATE {\color{blue}\texttt{\# on local machine}}
			\FOR{$i$ in $1, \cdots, n/2$}
			\STATE $v_i = \mathcal{S}_i(X_i,y_i)$. 
			\ENDFOR
			\STATE {\color{blue}\texttt{\# $\lceil\log d\rceil$ round communication}}
			\STATE  $\{\widehat{v}_1,\cdots, \widehat{v}_s\}$ = \texttt{HeavyHitter}($\{v_i\}_{i=1}^{n/2}, \varepsilon$, $\rho$).
			\STATE {\color{red}\texttt{\# coefficient estimation}}
			\STATE {\color{blue}\texttt{\# $n\wedge \sqrt{nm\varepsilon^2}$ round communication}}
			\STATE $\widehat{\beta}$ = \texttt{ULDPSCO}($\{(X_i, y_i)\}_{i=n/2+1}^{n}$, $T$, $\tau$, $\varepsilon$). 
			\STATE $\beta^{\widehat{v}_1 : \widehat{v}_s} = \widehat{\beta}$.
			\STATE {\bfseries Output: }{$\beta$}.
		\end{algorithmic}
	\end{algorithm}

	\subsubsection{Proof of Theorem \ref{thm:preciseestimation2}}

	We need the following technical result which states the effectiveness of optimization procedures in Algorithm \ref{alg:sco}. 
	
	\begin{lemma}\label{lem:convergenceofsco}
		Consider the stochastic convex optimization problem \eqref{equ:scoobject}. 
		Let $T=n\wedge \sqrt{nm\varepsilon^2 }$, and $\left\{\eta_t, \gamma_t\right\}_{t \in[T]}$ as in Lemma \ref{lem:convergenceofsco}, $L = 6s^3 \log n $, and $\tau \asymp L \sqrt{\log n \log \left(n\vee m\right) \log T / m}$.
		Then Algorithm \ref{alg:sco} is $\varepsilon$-ULDP and has
		\begin{align*}
			\mathbb{E}\left[\left\|{\beta}^{ag}_T - \widehat{\beta}^*\right\|_2^2\right]\lesssim \mathbb{E}\left[F\left( {\beta}_{T}^{ag}\right)-  F(\widehat{\beta}^*)\right] \lesssim\frac{s^9 \log^6 n}{nm\varepsilon^2  }+\frac{s^4 \log n }{n m} .
		\end{align*}
	\end{lemma}

	\begin{proof}[\textbf{Proof of Lemma \ref{lem:convergenceofsco}}]
		The privacy guarantee comes from the privacy of Algorithm \ref{alg:uldpmean} and the fact that each batch of samples are disjoint. 
		For the privacy guarantee, consider the mean squared error 
		\begin{align*}
			F(\beta)= 	\int_{\widehat{\mathcal{X}}\times \mathcal{Y}} \left(x^{\top}\beta - y\right)^2 d \mathrm{P}(x,y) = \mathbb{E}\left[\sigma^2 \right] + \left(\beta - \beta^*\right)^{\top} \Sigma \left(\beta - \beta^*\right). 
		\end{align*}
		By assumption on $\Sigma = \mathbb{E}\left(XX^{\top}\right)$, we have
		\begin{align*}
			F(\beta) - \inf_{\beta} F(\beta) = F(\beta) -  F(\widehat{\beta}^*)=  \left(\beta - \beta^*\right)^{\top}  \Sigma \left(\beta - \beta^*\right) \geq C_X^{-1} \|\beta-\beta^*\|_2^2. 
		\end{align*}
		Thus, it suffices to bound $\mathbb{E}\left[F(\beta) - F(\widehat{\beta}^*)\right]$ and the estimation error is of the same order.
		Note that under the assumption $\|\beta^*\|_2\leq 1$, the squared loss function $\ell(\beta)$ constraint on the unit ball has 
		\begin{align*}
			\|	\nabla \ell(\beta)\|_2 \leq \|(x^{\top}\beta - y)x \|_2 \lesssim s^3\log n
		\end{align*}
		i.e. $ s^3\log n$-Lipschitzness.
		Let $\left(\beta_1^{a g}, \ldots, \beta_T^{a g}\right)$ be the parameter trajector of Algorithm \ref{alg:sco}.
		Let $\left(\beta_1^{\prime a g}, \ldots, \beta_T^{\prime a g}\right)$ be the parameter trajectory of another algorithm  which replaces the gradient estimate $\tilde{\nabla} F\left(\theta_t^{\text {md }}\right)$ by
		\begin{align*}
			\tilde{\nabla} F^{\prime}\left(\beta_t^{m d}\right) \sim \frac{1}{n_0} \sum_{i \in S_{t,1} \cup S_{t,2}} g_i\left(\beta_t^{m d}\right)+\mathrm{Lap}\left(0, \frac{6\tau}{\varepsilon} {I}_d\right).
		\end{align*}
		By analysis analogous to the proof of  Lemma \ref{lem:privacyandutilityofuldpmean}, if we take $\tau \asymp L \sqrt{\log n \log \left(n\vee m\right) \log T / m}$, there holds
		\begin{align*}
			\beta_t^{ag} \stackrel{\mathcal{D}}{=}  \beta_t^{'ag}
		\end{align*}
		with probability $1 - 1/ nm$ for all $1\leq t \leq T$, where $\stackrel{\mathcal{D}}{=}$ stands for equal in distribution.
		Hence we have
		\begin{align}\label{equ:scoproof1}
			\mathbb{E}\left[F\left(\beta_T^{a g}\right)\right] \leq \mathbb{E}\left[F\left(\beta_T^{\prime a g}\right)\right] + \frac{s^4 \log n}{nm} .
		\end{align}
		For  $\mathbb{E}\left[F\left(\beta_T^{\prime a g}\right)\right]$, we use the fact that $\mathbb{E}\left[\tilde{\nabla} F^{\prime}\left(\beta_t^{m d}\right)\right]=\nabla F\left(\beta_t^{m d}\right)$ and, by Lemma \ref{lem:privacyandutilityofuldpmean}, 
		\begin{align*}
			\mathbb{E}\left[\left\|\tilde{\nabla} F^{\prime}\left(\beta_t^{m d}\right)-\nabla F\left(\beta_t^{m d}\right)\right\|_2^2\right] \lesssim 	\frac{L^2 s^2 \log ^3 n \log T}{n_0 m \varepsilon^2} + \frac{s \log n }{n_0 m}. 
		\end{align*}
		Applying Lemma \ref{lem:noisygradient}, we have
		\begin{align*}
			\mathbb{E}\left[F\left(\beta_T^{\prime ag }\right)-\min _{\beta } F(\beta) \right] \lesssim \frac{s^9 \log^5 n \log T}{nm\varepsilon^2  } + \frac{s^2 \log n }{n m}+ \frac{s}{T^2}. 
		\end{align*}
		Taking $T \asymp n\wedge \sqrt{nm\varepsilon^2}$, this together with \eqref{equ:scoproof1} lead to 
		\begin{align*}
			\mathbb{E}\left[F\left(\beta_T^{ ag }\right)- F(\widehat{\beta}^*) \right]  \lesssim \frac{s^9 \log^6 n}{nm\varepsilon^2  }+ \frac{s^4 \log n}{nm}.
		\end{align*}
	\end{proof}

	\begin{theorem}[\textbf{Formal version of Theorem \ref{thm:preciseestimation2}}]\label{thm:formalpreciseestimation2}
		Let data $\{(X_i, y_i)\}_{i=1}^n$ be generated as in \eqref{equ:modelassumption}. 
		Suppose $\{\mathcal{S}_i\}_{i=1}^{n}$ are $\alpha$-good selectors with $\alpha \gtrsim s^* \sqrt{\log n \log d/ n\varepsilon^2} $. 
		Suppose we let $\alpha / 8s^* \leq \rho \leq \alpha / 4 s^*$, $T=n\wedge \sqrt{nm\varepsilon^2 }$, and $\left\{\eta_t, \gamma_t\right\}_{t \in[T]}$ as in Lemma \ref{lem:convergenceofsco}, $L = 6s^3 \log n $, $\tau \asymp L \sqrt{\log n \log \left(n\vee m\right) \log T / m}$. 
		Let ${\beta}$ be the output of Algorithm  \ref{alg:uldpsparselinearregressionmultiround}. 
		Then we have (\textit{i}) Algorithm \ref{alg:uldpsparselinearregressionmultiround} is $\varepsilon$-ULDP. (\textit{ii}) there holds
		\begin{align*}
			\mathbb{E}\left[\left\|\beta^* - \beta\right\|_2^2\right] \lesssim  \frac{s^9 \log^6 n}{nm\varepsilon^2  }+ \frac{s^4 \log n}{nm}.
		\end{align*}
	\end{theorem}

	\begin{proof}[\textbf{Proof of Theorem \ref{thm:formalpreciseestimation2}}]
		By Lemma \ref{lem:bassilyheavyhitter} and \ref{lem:convergenceofsco}, both \texttt{HeavyHitter} and \texttt{ULDPSCO} are $\varepsilon$-ULDP. 
		Since their associated users do not cross, we have Algorithm \ref{alg:uldpsparselinearregressionmultiround} is also $\varepsilon$-ULDP.
		As for (\textit{ii}), by Proposition \ref{prop:heavyhitterselection}, we know that all the non-zero variables of $\beta^*$ is included in $\{\widehat{v}_1,\cdots, \widehat{v}_s\}$ with probability $1 - 1 / n^2$.
		Thus, we have 
		\begin{align*}
			\left\|\beta^*-\beta\right\|_2^2 = \left\|\widehat{\beta}^*-\widehat{\beta}\right\|_2^2.
		\end{align*}
		Applying Lemma \ref{lem:convergenceofsco}, this leads to 
		\begin{align*}
			\mathbb{E}\left[\left\|\beta^*-\beta\right\|_2^2\right]\lesssim	\mathbb{E}\left[\left\|{\beta}^{ag}_T - \widehat{\beta}^*\right\|_2^2\right] \lesssim\frac{s^9 \log^6 n}{nm\varepsilon^2  }+\frac{s^4 \log n }{n m}  + \frac{1}{n^2}\lesssim \frac{s^{*9} \log^6 n}{nm\varepsilon^2 \alpha^9}+\frac{s^{*4} \log n }{n m \alpha^4} ,
		\end{align*}
		where in the last step we used $s \lesssim s^* / \alpha$ as in Proposition \ref{prop:heavyhitterselection}. 
		The additional term $1 / n^2$ is due to the failure probability of Proposition \ref{prop:heavyhitterselection} and is omitted since it is adjustable to any level with a constant multiplicative cost on the other terms. 
	\end{proof}

	\subsection{The Two Round Protocol}

	\subsubsection{Proof of Proposition \ref{prop:distributionofbeta}}

	\begin{proof}[\textbf{Proof of Proposition \ref{prop:distributionofbeta}}]
		For the first conclusion, consider th local OLS estimator on selected variables of user $i$, which is $\widehat{\beta} = (\widehat{X}_{i}^{\top}\widehat{X}_{i})^{-1} \widehat{X}_{i}^{\top} y_i$.
		Given the fact that $m \geq s$, $\widehat{X}_i^{\top}\widehat{X}_i$ is invertible and we have
		\begin{align*}
			\widehat{\beta}_i = (\widehat{X}_i^{\top}\widehat{X}_i)^{-1}  \widehat{X}_i^{\top} \widehat{y}_i = (\widehat{X}_i^{\top}\widehat{X}_i)^{-1}  \widehat{X}_i^{\top} ( \widehat{X}_i \widehat{\beta}^* + \sigma_i ) = \widehat{\beta}^* + (\widehat{X}_i^{\top}\widehat{X}_i)^{-1}  \widehat{X}_i^{\top} \sigma_i, 
		\end{align*}
		where $\sigma_{i,j}$ are i.i.d. sub-Gaussian random variables for $1\leq j \leq m$. 
		Therefore, the first argument follows from
		\begin{align*}
			\mathbb{E}[\widehat{\beta}_i]  = \widehat{\beta}^* + (\widehat{X}_i^{\top}\widehat{X}_i)^{-1}  \widehat{X}_i^{\top} \mathbb{E}[\sigma_i] = \widehat{\beta}^*. 
		\end{align*}
		By implication of  \citet[Theorem 2.1]{hsu2012tail}, we have
		\begin{align*}
			\mathrm{Pr}\left(\|\widehat{\beta}_i - \widehat{\beta}^*\|_2 \geq  \sqrt{3 \log n \cdot \mathrm{tr}\left[\left(\widehat{X}_i^{\top}\widehat{X}_i\right)^{-1}\right]  \cdot \mathbb{E}[\sigma_{i,j}^2]}  \right) \leq 1 - \frac{1}{n^3}.
		\end{align*}
		This together with covariance matrix estimation bounds ( e.g. \citet[Theorem 6.5]{wainwright2019high}) lead to
		\begin{align}\label{equ:distributionofbeta1}
			\|\widehat{\beta}_i - \widehat{\beta}^*\|_2 \lesssim  \sqrt{\frac{\mathrm{tr}[\widehat{\Sigma}^{-1}]  \log n }{m}} \lesssim \sqrt{\frac{s\log n}{m}}
		\end{align}
		with probability $1 - 1 / n^3$. 
		Applying union bound, \eqref{equ:distributionofbeta1} holds for all $i = n / 2 + 1 , \cdots, n$ with probability at least $1 -1  / n^2$. 
		For the second statement, if either conditions in Proposition \ref{prop:existenceofgoodselectors} holds, we can adopt Lasso (or SCAD) on the selected variables.
		See Example \ref{example:lasso} and \ref{example:scad}.  
		The oracle results in \citet{belloni2013least} (or \citet{fan2011nonconcave}) yield the concentration bound with the true sparsity parameter
		\begin{align*}
			\|\widehat{\beta}_i - \widehat{\beta}^*\|_2 \lesssim  \sqrt{\frac{s^*\log n}{m}}
		\end{align*}
		for all $i = n / 2 + 1 , \cdots, n$ with probability at least $1 -1  / n^2$. 
	\end{proof}

	\subsubsection{ULDP Mean Estimation} \label{app:uldpmeanestimation}
	
	We borrow the idea from \citet{girgis2022distributed} while slight modifications are made. 
	The estimation is conducted in two stages. 
	In the first stage, a histogram partition of $\widehat{\mathcal{X}}$ with bin width $\sqrt{ \log^2 n / m}$ is created. 
	The server privately estimates the range in which the means $\widehat{\beta}_i$ lie with high probability (Algorithm \ref{alg:rangescalar}). 
	In the second stage, each user projects its $\widehat{\beta}_i$ into the determined range from the first step. 
	Then, all users send the LDP versions of their projected $\widehat{\beta}_i$ to the curator (Algorithm \ref{alg:meanscalar}).
	Both steps are scalar operations. 
	In the vector case, instead of applying them to each dimension separately, random rotation \citep{levy2021learning} is adopted to eliminate a superfluous factor of $\mathcal{O}(\sqrt{s})$. 
	The full algorithm is summarized in Algorithm \ref{alg:uldpmean}.  
	We only consider pure differential privacy here and utilize Laplace noise instead of Gaussian in \citet{girgis2022distributed}.

	\begin{algorithm}[H]
		\caption{\texttt{Range} }
		\label{alg:rangescalar}
		\begin{algorithmic}
			\STATE {\bfseries Input: }{ Scalars $\{y_i\}$, concentration radius $\tau$,  privacy budget $\varepsilon$. } \\
			\STATE \texttt{\# user side} \\
			\STATE All users divide the interval $[-1, 1]$ into $k=1 / \tau$ disjoint intervals, each with width $2 \tau$. Let $\mathcal{T}:=\{a_1, a_2, \ldots, a_k\}$ be the index set of middle points of intervals.
			\FOR{$y$ in $\{{y}_i\}$}
			\STATE Compute $\nu=\arg \min _{a_j \in \mathcal{T}}\left|y-a_j\right|$.
			\STATE Uniformly sample $j \in [k]$.
			\STATE	Compute $p = {H}_k^{\top j}\cdot  e_\nu / \sqrt{k}$, where $e_\nu$ denotes the basis vector corresponding to $\nu$ and ${H}_k$ is a size $k$ Hadamard matrix.
			\STATE Compute vector $z_i$ :
			\begin{align*}
				{z}_i= \begin{cases}+{H}_k^{\top j}\cdot \frac{e^{\varepsilon}+1}{e^{\varepsilon}-1} & \text { w.p. } \frac{1}{2}+\frac{\sqrt{k} \cdot p}{2} \frac{e^{\varepsilon}-1}{e^{\varepsilon}+1} \\ -{H}_k^{\top j}\cdot \frac{e^{\varepsilon}+1}{e^{\varepsilon}-1} & \text { w.p. } \frac{1}{2}-\frac{\sqrt{k} \cdot p}{2} \frac{e^{\varepsilon}-1}{e^{\varepsilon}+1}\end{cases}
			\end{align*}
			\ENDFOR
			\STATE \texttt{\# curator side} \\
			\STATE $\overline{z} = \sum z_i$ and $\ell = {\arg \max }_j \overline{z}^j$.\\
			{\bfseries Output: }{Bin $[a_{\ell} - 3 \tau, a_{\ell} + 3 \tau]$.}
		\end{algorithmic}
	\end{algorithm} 
	
	Let the standard Laplace random variable have probability density function $e^{-|x|} / 2$ for $x\in \mathbb{R}$. 
	
	\begin{algorithm}[H]
		\caption{\texttt{Mean} }
		\label{alg:meanscalar}
		\begin{algorithmic}
			\STATE {\bfseries Input: }{ Scalars $\{{y}_i\}_{i=1}^n$, concentration range $[a, b]$,  privacy budget $\varepsilon$. } \\
			\STATE \texttt{\# user side} \\
			\FOR{$i$ in $1, \cdots, n$}
			\STATE  Let $\tilde{y}_i = \Pi_{[a,b]} y_i + \textrm{Lap}(0,  |b - a|  / \varepsilon)$,   where $\Pi_{[a,b]}$ is the projection onto $[a,b]$.
			\ENDFOR
			\STATE \texttt{\# curator side} \\
			{\bfseries Output: }{$\sum \tilde{y}_i / n$.}
		\end{algorithmic}
	\end{algorithm}

	\begin{algorithm}[H]
		\caption{\texttt{ULDPMean} }
		\label{alg:uldpmean}
		\begin{algorithmic}
			\STATE {\bfseries Input: }{ Two groups of local coefficients $\mathcal{B}_1 = \{{\beta}_i\}_{i = 1}^{n / 2}$ and $\mathcal{B}_2 = \{{\beta}_i\}_{i = n / 2}^{n}$, concentration radius $\tau$,  privacy budget $\varepsilon$. } 
			\STATE {\bfseries Initialization: }   Let $D = \mathrm{Diag}(w)$ and $U = H_s D / \sqrt{s}$, where $w_i\sim \mathrm{Unif}\{-1,1\}$ and $H_s$ is a size $s$ Hadamard matrix. 
			Let $z$ be a $s$ dimensional zero vector. 
			\STATE \texttt{\# histogram selection} 
			\FOR{$\ell$ in $1,\cdots, s$}
			\FOR{$\beta_i$ in $\mathcal{B}_1$}
			\STATE $y_{\ell,i} = (U \beta_i)^{\ell}$.
			\ENDFOR
			\STATE $R_{\ell} = \mathtt{Range}(\{y_{\ell, i}\}_{i=1}^{n / 2}, \tau, \varepsilon / s)$.
			\ENDFOR
			\STATE \texttt{\# coefficient estimation} 
			\FOR{$\beta_i \in \mathcal{B}_2$}
			\STATE $\ell = i \text{ mod } s$.
			\STATE $y_{\ell,i} = (U \beta_i)^{\ell}$.
			\ENDFOR
			\FOR{$j$ in $1, \cdots, s$}
			\STATE $z^j = s \cdot \texttt{Mean}(\{y_{\ell,i} \text{ such that } \ell = j\}, R_{j}, \varepsilon)$.
			\ENDFOR
			\STATE {\bfseries Output: }{$U^{-1}z$.}
		\end{algorithmic}
	\end{algorithm} 
	
	The following lemma is a modified version of Theorem 2 of \citet{girgis2022distributed} under pure differential privacy. 
	
	\begin{lemma}\label{lem:themodifiedmeanconcentration}
		Let $\widehat{\beta}^*$ be the true underlying coefficient, and $\widehat{\beta}_i$s be the coefficients estimated by each user. 
		Suppose $\mathbb{E}\left[\widehat{\beta}_i\right] = \beta^*$ and $\|\widehat{\beta}_i - \beta^*\|_2 \leq \tau$ with probability $1 - 1 / n^2$ for all $i$. 
		Then with probability $1 - 1 / n^2$,  we have 
		\begin{align*}
			\left\| \frac{2}{n}\sum_{i=n/2+1}^n{\widehat{\beta}}_i - \mathtt{ULDPMean}(\{\widehat{\beta}_i\}_{i=1}^{n/2}, \{\widehat{\beta}_i\}_{i=n/2+1}^{n},  \tau, \varepsilon)\right\|^2_2 \lesssim \frac{s \tau^2 \log^2 n}{n\varepsilon^2} 
		\end{align*}
	\end{lemma}
	
	\begin{proof}[\textbf{Proof of Lemma \ref{lem:themodifiedmeanconcentration}}]
		We know the $\widehat{\beta}_i$s satisfy Definition 2 in \citet{girgis2022distributed} with parameter $(\tau, {1}/{n^2})$. 
		By \citet{levy2021learning}, we have 
		\begin{align*}
			\|U\widehat{\beta}_i - U \widehat{\beta}^*\|_{\infty} \lesssim \sqrt{\frac{\tau^2 \log sn^2 }{s}}. 
		\end{align*}
		If we choose $\tau' \asymp \sqrt{{\tau^2 \log sn^2 }/ {s}}\asymp \sqrt{\tau^2 {\log n}/{s}}$, then $y_i$ satisfy Definition 2 in \citet{girgis2022distributed} with parameter $(\tau', {1}/{n^2})$. 
		Then the Lemma 1 of \citet{girgis2022distributed} implies that $\Pi_{[a,b]} y_i = y_i$ with probability $1 - 1 / n^2$ in Algorithm \ref{alg:meanscalar}. 
		Then the least square error for \texttt{Mean} is 
		\begin{align*}
			\left|\frac{2s}{n}\sum_{i=1}^{n/2s}\tilde{y}_i - \frac{2s}{n}\sum_{i=1}^{n/2s}y_i \right|^2 = \left|\frac{|b - a|s}{n\varepsilon}\sum_{i=1}^{n/2s} \gamma_i\right| \leq \sqrt{\frac{288 s \tau^{\prime 2} \log n}{n\varepsilon^2}}
		\end{align*}
		where the inequality follows from (2.18) in \citet{wainwright2019high}. 
		Since $\|\cdot\|_2$ is upper bounded by $\sqrt{s}$ times infinity norm, there holds
		\begin{align*}
			& \left\| \frac{2}{n}\sum_{i=n/2+1}^n{\widehat{\beta}}_i - \mathtt{ULDPMean}(\{{\widehat{\beta}}_i\}_{i=1}^{n/2}, \{{\widehat{\beta}}_i\}_{i=n/2+1}^n,  \tau, \varepsilon)\right\|^2_2 \\ 
			=&  \left\|\frac{2}{n}U\sum_{i=n/2+1}^n{\widehat{\beta}}_i  - z\right\|_2^2
			\leq s \cdot \left\| \frac{2}{n}U\sum_{i=n/2+1}^n{\widehat{\beta}}_i  - z\right\|_{\infty}^2 \leq \frac{288 s^2 \tau^{\prime 2} \log n}{n\varepsilon^2} \lesssim \frac{s \tau^2 \log^2 n}{n\varepsilon^2}. 
		\end{align*}
	\end{proof}
	
	The following lemma is the key technical result to prove Theorem \ref{thm:preciseestimation}. 
	
	\begin{lemma}[\textbf{Privacy and utility of Algorithm \ref{alg:uldpmean}}]
		\label{lem:privacyandutilityofuldpmean}
		Let $\widehat{\beta}^*$ be the true underlying coefficient, and $\widehat{\beta}_i$s be the coefficients estimated by each user. 
		Then the algorithm \ref{alg:uldpmean} is $\varepsilon$-ULDP. 
		Moreover, there exists some $\tau \asymp  \sqrt{{\log^2 n}/{m}}$ such that, with probability $1 - 2 / n^2$,  we have 
		\begin{align*}
			\left\| \widehat{\beta}^* - \mathtt{ULDPMean}(\{\widehat{\beta}_i\}_{i=1}^{n/2}, \{\widehat{\beta}_i\}_{i=n/2+1}^{n},  \tau, \varepsilon)\right\|^2_2 \lesssim \frac{s^2 \log^3 n  }{nm\varepsilon^2} + \frac{s \log n  }{n m} 
		\end{align*}
	\end{lemma}

	\begin{proof}[\textbf{Proof of Lemma \ref{lem:privacyandutilityofuldpmean}}]
		We first show the privacy property of Algorithm \ref{alg:uldpmean}. 
		Since the users of \texttt{Range} and \texttt{Mean} do not across, it suffices to show that both of the algorithms are $\varepsilon$-ULDP. 
		The privacy of \texttt{Range} follows from Lemma 1 of \citet{girgis2022distributed}. 
		The privacy of \texttt{Mean} is straightforward by property of Laplace mechanism, given that the sensitivity of $\Pi_{[a,b]}y$ is $|b - a|$. 
		Now we prove the accuracy part. 
		The squared error can be decomposed into two parts associating to private estimation error and non-private estimation error, respectively.
		\begin{align*}
			&\left\| \widehat{\beta}^* - \mathtt{ULDPMean}(\{\widehat{\beta}_i\}_{i=1}^{n/2}, \{\widehat{\beta}_i\}_{i=n/2+1}^{n},  \tau, \varepsilon)\right\|^2_2\\ \leq &  2 \cdot \left(\left\| \frac{2}{n}\sum_{i=n/2+1}^n{\widehat{\beta}}_i - \mathtt{ULDPMean}(\{{\widehat{\beta}}_i\}_{i=1}^{n/2}, \{{\widehat{\beta}}_i\}_{i=n/2+1}^n,  \tau, \varepsilon)\right\|^2_2 + \left\| \frac{2}{n}\sum_{i=n/2+1}^n{\widehat{\beta}}_i - \widehat{\beta}^* \right\|_2^2\right).
		\end{align*}
		We deal with private estimation error part first.
		From Proposition \ref{prop:distributionofbeta}, we know the $\widehat{\beta}_i$s satisfy Lemma \ref{lem:themodifiedmeanconcentration} with $\tau = \sqrt{s \log n / m}$. 
		Then we have
		\begin{align}
			\left\| \frac{2}{n}\sum_{i=n/2+1}^n{\widehat{\beta}}_i - \mathtt{ULDPMean}(\{{\widehat{\beta}}_i\}_{i=1}^{n/2}, \{{\widehat{\beta}}_i\}_{i=n/2+1}^n,  \tau, \varepsilon)\right\|^2_2 \lesssim \frac{s^2 \log^2 n \log n}{nm\varepsilon^2}. \label{equ:meanestimationestimation}
		\end{align}
		If either conditions in Proposition \ref{prop:existenceofgoodselectors} holds, the parameter becomes $\tau = \sqrt{s^* \log n / m}$ by Proposition \ref{prop:distributionofbeta}, and the same analysis goes with $s^*$ instead of $s$. 
		\begin{align}
			\left\| \frac{2}{n}\sum_{i=n/2+1}^n{\widehat{\beta}}_i - \mathtt{ULDPMean}(\{{\widehat{\beta}}_i\}_{i=1}^{n/2}, \{{\widehat{\beta}}_i\}_{i=n/2+1}^n,  \tau, \varepsilon)\right\|^2_2 \lesssim \frac{s s^* \log^2 n \log n}{nm\varepsilon^2}. \label{equ:meanestimationestimationsstar}
		\end{align}
		Next, we bound the non-private estimation error. 
		When $\widehat{\beta}_i$ is the OLS estimator, by its sub-Gaussianality, we have
		\begin{align}\label{equ:meanestimationapprox}
			\left\| \frac{2}{n}\sum_{i=n/2+1}^n{\widehat{\beta}}_i - \widehat{\beta}^* \right\|_2^2 \lesssim \frac{s \log n}{n m}. 
		\end{align}
		If either conditions in Proposition \ref{prop:existenceofgoodselectors} holds, this becomes
		\begin{align}\label{equ:meanestimationapproxsstar}
			\left\| \frac{2}{n}\sum_{i=n/2+1}^n{\widehat{\beta}}_i - \widehat{\beta}^* \right\|_2^2\lesssim \frac{s^* \log n}{n m}. 
		\end{align}
		Together, \eqref{equ:meanestimationestimation} and \eqref{equ:meanestimationapprox} lead to 
		\begin{align*}
			\left\| \widehat{\beta}^* - \mathtt{ULDPMean}(\{\widehat{\beta}_i\}_{i=1}^{n/2}, \{\widehat{\beta}_i\}_{i=n/2+1}^{n},  \tau, \varepsilon)\right\|^2_2 \lesssim \frac{s^2 \log^3 n}{nm\varepsilon^2} + \frac{s \log n}{n m}.
		\end{align*}
		The overall failure probability is at least $2 /n^2$ since we utilized two high probability arguments. 
		Similarly,  \eqref{equ:meanestimationestimationsstar} and \eqref{equ:meanestimationapproxsstar} lead to 
		\begin{align*}
			\left\| \widehat{\beta}^* - \mathtt{ULDPMean}(\{\widehat{\beta}_i\}_{i=1}^{n/2}, \{\widehat{\beta}_i\}_{i=n/2+1}^{n},  \tau, \varepsilon)\right\|^2_2 \lesssim \frac{ss^* \log^3 n}{nm\varepsilon^2} + \frac{s^* \log n}{n m}.
		\end{align*}
	\end{proof}

	\subsubsection{Proof of Theorem \ref{thm:preciseestimation}}

	\begin{proof}[\textbf{Proof of Theorem \ref{thm:preciseestimation}}]
		By Lemma \ref{lem:bassilyheavyhitter} and \ref{lem:privacyandutilityofuldpmean}, both \texttt{HeavyHitter} and \texttt{ULDPMean} are $\varepsilon$-ULDP. 
		Since their associated users do not cross, we have Algorithm \ref{alg:uldpsparselinearregression} is also $\varepsilon$-ULDP.
		As for (\textit{ii}), by Proposition \ref{prop:heavyhitterselection}, we know that all the non-zero variables of $\beta^*$ is included in $\{\widehat{v}_1,\cdots, \widehat{v}_s\}$ with probability $1 - 1 / n^2$.
		Thus, we have 
		\begin{align*}
			\left\|\beta^*-\beta\right\|_2^2 = \left\|\widehat{\beta}^*-\widehat{\beta}\right\|_2^2.
		\end{align*}
		Applying Lemma \ref{lem:privacyandutilityofuldpmean}, this leads to 
		\begin{align*}
			\left\|\beta^*-\beta\right\|_2^2 \lesssim \frac{s^2 \log^3 n}{nm\varepsilon^2} + \frac{s \log n}{n m} \lesssim \frac{s^{*2} \log^3 n}{nm\varepsilon^2 \alpha^2} + \frac{s^* \log n}{n m \alpha } ,
		\end{align*}
		where in the last step we used Proposition \ref{prop:heavyhitterselection}. 
		In the last, the overall failure probability of Proposition \ref{prop:heavyhitterselection}, Lemma \ref{lem:bassilyheavyhitter}, and Lemma \ref{lem:privacyandutilityofuldpmean} is at most $ 4 / n^2$. 
	\end{proof}

	\section{Extension to Sparse Estimation}\label{app:extensiontosparseestimation}

	The full statement of Theorem \ref{thm:preciseestimationgeneral} is as follows.
	We utilize Algorithm \ref{alg:uldpsparselinearregression} while modifying the estimators $\widehat{\beta}_i$ and selectors $\mathcal{S}_i$ to accommodate the general problem.

	\begin{theorem}[\textbf{Formal version of Theorem \ref{thm:preciseestimationgeneral}}]\label{thm:preciseestimationgeneralfull}
		Let data $\{X_i\}_{i=1}^n$ be generated by $\mathrm{P}_{\beta^*}$ for $\beta^* \in \Omega_{s,a}^d$.
		Suppose we have non-private estimators:
		(\textit{i})
		estimator $\tilde{\beta}_i$ with $\|\tilde{\beta}_i - {\beta}^*\|_2 \leq \nu_1$ for all $ 1\leq i \leq n / 2$ and 
		(\textit{ii}) 
		estimator $\widehat{\beta}_i$ on selected variables with $\mathbb{E}\left[\widehat{\beta}_i\right] = \widehat{\beta}^*$ and $\|\widehat{\beta}_i - \widehat{\beta}^*\|_2 \leq \nu_2$ for all $n /2 + 1\leq i \leq n$. 
		Then there exist $\alpha$-good selectors $\{\mathcal{S}_i\}_{i=1}^{n/2}$ with $\alpha\gtrsim s^* \sqrt{\log n \log d/ n\varepsilon^2} $, that is $\mathrm{Pr}\left(v = \mathcal{S}_i(X_i)\right) \geq \alpha / s^*$ for $1\leq v\leq s^*$ and $1 \leq i \leq n/2$. 
		Suppose we let $\alpha / 8s^* \leq \rho \leq \alpha / 4 s^*$, $\tau \asymp \sqrt{\nu^2 \alpha \log n / s^*}$.
		Then, for any $a \gtrsim\nu_1$,
		Algorithm \ref{alg:uldpsparselinearregression} is $\varepsilon$-ULDP and has an output $\beta$ with
		\begin{align}\label{equ:generalboundappendix1}
			\left\|\beta^* - \beta\right\|_2^2 \lesssim   \frac{\nu_2^2}{n}+ \frac{\nu_2^2 s^{*} \log^2 n }{n \varepsilon^2\alpha }
		\end{align}
		with probability at least $1 - 3 / n^2$. 
		Moreover, for $\ell_1$ norm, there holds
		\begin{align}\label{equ:generalboundappendix2}
			\left\|\beta^* - \beta\right\|_1 \lesssim   \sqrt{\frac{\nu_2^2 s^*}{n \alpha}}+ \sqrt{\frac{\nu_2^2 s^{*2 } \log^2 n }{n \varepsilon^2\alpha^2 }}
		\end{align}
		with probability at least $1 - 3 / n^2$. 
		
	\end{theorem}

	\begin{proof}[\textbf{Proof of Theorem \ref{thm:preciseestimationgeneralfull}}]
		The privacy guarantee follows from Theorem \ref{thm:preciseestimation}. 
		Since $a \gtrsim \nu_1$, we can consistently select all true variables with proxy estimators. 
		This implies we can have $\alpha$-good selectors with $\alpha \gtrsim 1 \gtrsim s^* \sqrt{\log n \log d/ n\varepsilon^2} $. 
		By Proposition \ref{prop:heavyhitterselection}, we know that all the non-zero variables of $\beta^*$ is selected with probability $1 - 1 / n^2$.
		Thus, we have $\left\|\beta^*-\beta\right\|_2^2 = \left\|\widehat{\beta}^*-\widehat{\beta}\right\|_2^2$. 
		Applying Lemma \ref{lem:themodifiedmeanconcentration}, we have
		\begin{align*}
			\left\|\widehat{\beta}^*-\widehat{\beta}\right\|_2^2 \lesssim  	\left\|\widehat{\beta}^*- \frac{2}{n} \sum_{i=n / 2+1}^n \hat{\beta}_i\right\|_2^2 +  \frac{s\nu_2^2 \log^2 n}{n\varepsilon^2}. 
		\end{align*}
		Since $\widehat{\beta}_i$ are concentrated, it is sub-Gaussian.
		Thus, there holds
		\begin{align*}
			\left\|\widehat{\beta}^*-\widehat{\beta}\right\|_2^2 \lesssim  \frac{\nu_2^2}{n} +  \frac{s\nu_2^2 \log^2 n}{n\varepsilon^2} \lesssim  \frac{\nu_2^2}{n} +  \frac{s^*\nu_2^2 \log^2 n}{n\varepsilon^2 \alpha}
		\end{align*}
		where in the last step we used Proposition \ref{prop:heavyhitterselection}. 
		In the last, the overall failure probability of Proposition \ref{prop:heavyhitterselection}, Lemma \ref{lem:bassilyheavyhitter}, and Lemma \ref{lem:themodifiedmeanconcentration} is at most $ 3 / n^2$. 
		This yields \eqref{equ:generalboundappendix1}.
		For \eqref{equ:generalboundappendix2}, note that there is only $s \lesssim s^* / \alpha$ none zero elements.
		Using the difference between $\ell_1$ and $\ell_2$ norms, which is $\sqrt{s}$, yields \eqref{equ:generalboundappendix2}. 
	\end{proof}

	\section{Additional Experiment Results}\label{app:additionalexperiment}
	
	\subsection{Implementation Details}
	
	For each model, we report the best result over its parameter grids, with the best result determined based on the average result of at least 30 replications.
	We do not perform any parameter selection (e.g. cross validation or validation set) since they are prohibitive under locally private setting \citep{ma2024optimal, ma2024locally} or will cost too much privacy budget \citep{papernot2021hyperparameter}. 
	The parameter grids size are selected based on running time so that each method costs equal amount of computation. 
	Efficient methods receive a exhaustive parameter grid and can be properly tuned.
	Computation heavy methods receive a small grid with insensitive parameters set to default.

	\begin{itemize}
		\item For candidate variable selector of our methods, we adopt the Lasso estimator and identify its non-zero coefficients as the selected variables. 
		Moreover, we conduct a feature screening (see Appendix \ref{app:pluginselector} for detail) for acceleration. 
		The number of screened variables is set to 64.
		The number of selected variables $s$ is selected in $\{2, 4, 8, 16\}$.
		\begin{itemize}
			\item \textbf{2-SLR}: The two-round sparse linear regression protocol is implemented based on Algorithm \ref{alg:uldpsparselinearregression}. We select the range $[-B, B]$ in $B\in \{1,2,3\}$ and the concentration radius is decided by the number of bins, which is in $\{2,4,8,16,32\}$. 
			\item \textbf{M-SLR}: The multi-round sparse linear regression protocol is implemented based on Algorithm \ref{alg:uldpsparselinearregressionmultiround}. 
			We set $B =3$ and select the number of bins in $\{2,4,8,16,32\}$. 
			Moreover, we set the learning rate of the gradient to be $\eta_t = 0.1\cdot (\frac{1 + t}{2})^{0.2}$. 
		\end{itemize}
		\item \textbf{LDPPROX}: The non-interactive locally differentially private sparse linear regressor based on proxy estimator is implemented according to Algorithm 1 in \citet{zhu2023improved}. 
		Due to the heavy computation burden, we set $r = \sqrt{d \cdot \log n}$, $\tau_1 = 4$, $\tau_2 = 8$. 
		In simulation where we know $\min_{\beta^{*j} \neq 0} |\beta^{*j}| = 0.2$, we set $\lambda = 0.05$. 
		In real data, we set $\lambda$ to the 10-th lower quantile of the absolute fitted coefficients. 
		\item \textbf{LDPIHT}: The locally differentially private iterative hard thresholding is implemented according to Algorithm 2 in \citet{zhu2023improved}. 
		We select $T\in\{2, 5, 10, 20, 50\}$, $\eta\in\{0.01, 0.1, 1\}$, $\tau_1, \tau_2 \in  \{2, 4, 8\}$, $k' \in \{5, 10, 20, 50\}$. 
		\item \textbf{Lasso}: The conventional Lasso regressor is fitted using the \textit{LassoCV} class in scikit-learn package \citep{scikit-learn}.  We set $n\_alphas = 300$, $max\_iter = 3000$, and $tol = 10^{-4}$.
	\end{itemize}

	\subsection{Additional Simulation Results}
	
	We present the additional result of the correlated marginal distribution data experiment omitted in the main text due to page limitation.
	The correlation of the first 50 dimensions are set to be exponentially decaying, i.e. 
	\begin{align*}
		\mathrm{Cov}\left(X_{i,j}^k , X_{i,j}^{k'}\right) = 2^{- |k - k'|}.
	\end{align*}
	We draw each $\sigma_{i,j}$ correlatedly from a standard Gaussian distribution.
	For $\beta^*$, we randomly select $s^* = 8$ coordinates in the first 50 dimensions to be $0.2$ and let others be zero.
	We typically set $n = 400$, $m = 100$, $d = 256$, and $\varepsilon = 4$, while varying one of them to observe how the evaluated metric varies. 
	We use squared error as evaluation of the estimated coefficient and F1 score as evaluation of the selected variables. 
	
	\begin{figure*}[htbp]
		\centering
		\subfigure[$d$ - F1 score.]{
			\begin{minipage}{0.23\linewidth}
				\centering
				\includegraphics[width=\linewidth]{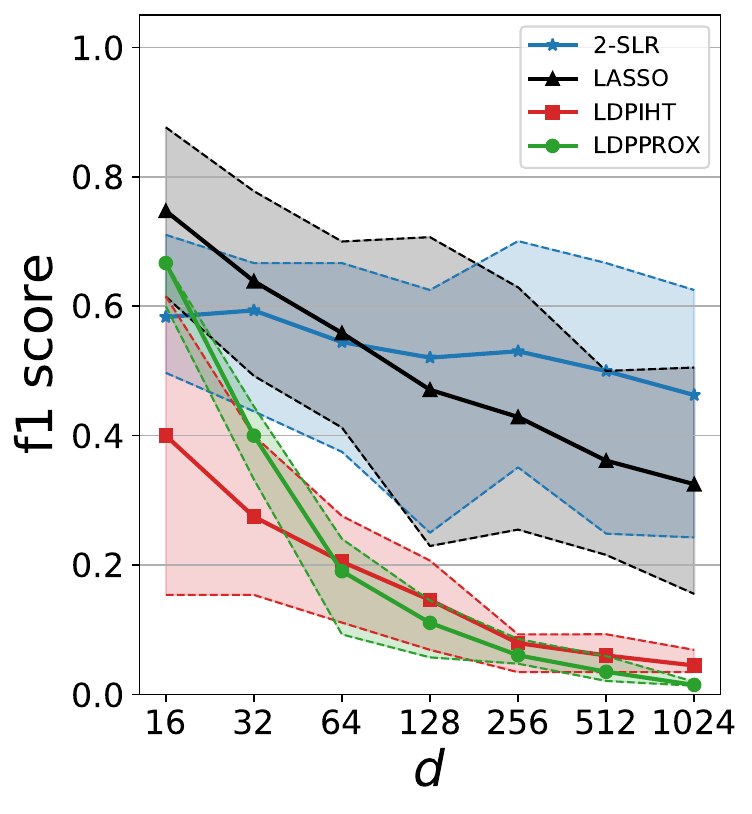}
			\end{minipage}
			\label{fig:df1cor}
		}
		\subfigure[ $d$ - $\ell_2$ error with $m = 100$.]{
			\begin{minipage}{0.23\linewidth}
				\centering
				\includegraphics[width=\linewidth]{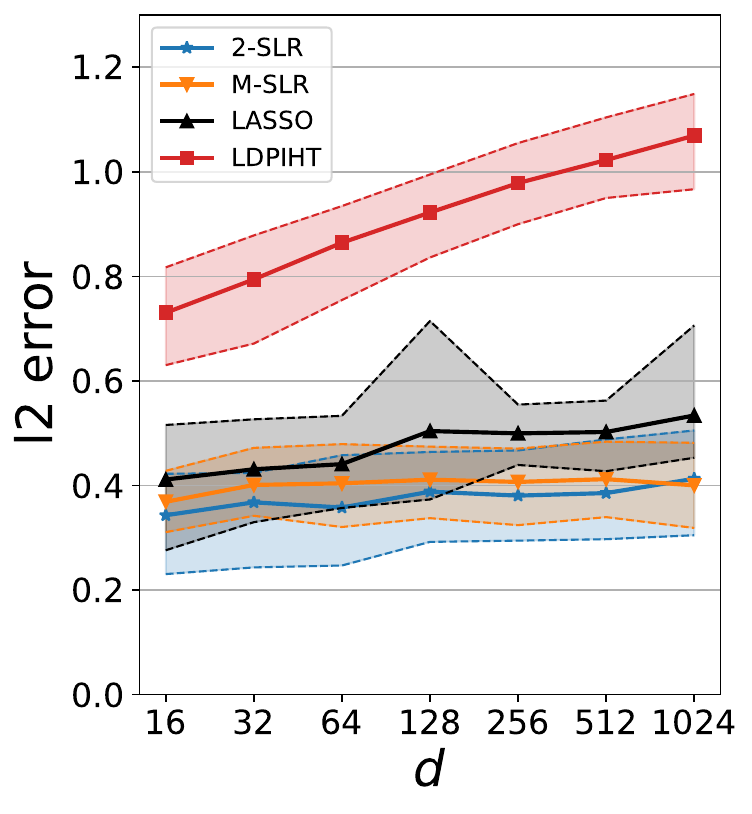}
			\end{minipage}
			\label{fig:dl2m100cor}
		}
		\subfigure[$d$ - $\ell_2$ error with $m = 200$.]{
			\begin{minipage}{0.23\linewidth}
				\centering
				\includegraphics[width=\linewidth]{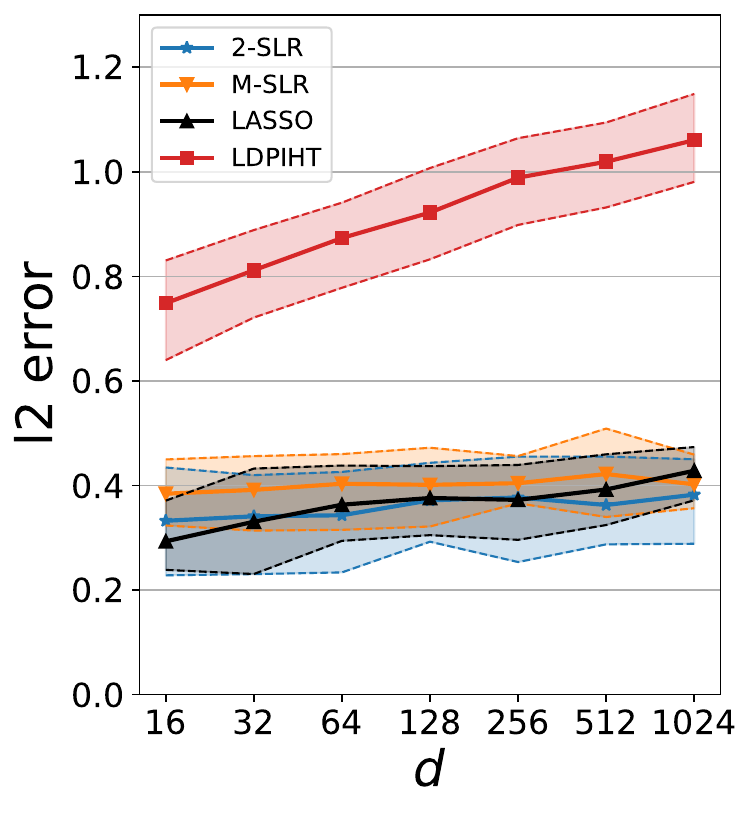}
			\end{minipage}
			\label{fig:dl2m200cor}
		}
		\subfigure[$\varepsilon$ - $\ell_2$ error.]{
			\begin{minipage}{0.23\linewidth}
				\centering
				\includegraphics[width=\linewidth]{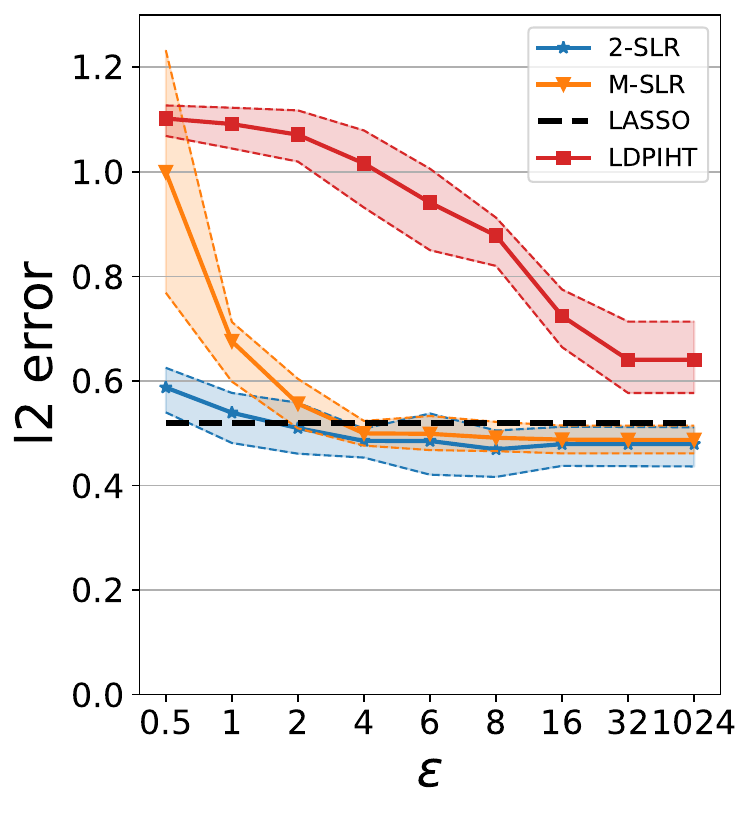}
			\end{minipage}
			\label{fig:epsilonl2cor}
		}
		\vskip -0.1in
		\caption{Experiments w.r.t. $d$ and $\varepsilon$ for correlated marginal. 
			We plot the quantiles over 30 repetitions with $95\%$ coverage. 
			We exclude LDPPROX in some figures since it is highly unstable and do not fit into our plot scale. 
		}
		\label{fig:epsilondcor}
	\end{figure*}
	
	We conduct experiments with respect to $d$. 
	We first analyze the variable selection performance.
	Due to the high sparsity, we use F1 score as the evaluation criterion.
	For $d\in\{16,32, \cdots, 1024\}$, we compute the averaged F1 scores of the proposed candidate variable selection (represented by 2-SLR) and other methods.
	As depicted in Fig. \ref{fig:df1cor}, the overall performance of all methods deteriorates compared to that under the independent setting, whereas 2-SLR remains stable and maintains its advantage. 
	When $d=16$, the selection performances of Lasso and LDPPROX are slightly superior than the variables induced by other methods. 
	However, as $d$ increases, the variable selection performance of Lasso, LDPIHT and LDPPROX decreases sharply, while the F1 scores of 2-SLR only fluctuate slightly and become higher than those of other competitors when $d \geq 64$.
	
	Then, we analyze the estimation performance.
	In \ref{fig:dl2m100cor} and \ref{fig:dl2m200cor}, we plot the curve of $\ell_2$ error with respect to $d$. 
	Whether $m=100$ or $m=200$, the proposed methods are less sensitive to $d$ compared to LDPIHT.
	The is also compatible with rate in \eqref{equ:tworoundbound2} which scales with $\log d$.
	Compared to the independent case, the correlated case requires more local samples to achieve a consistent selection.
	Thus, thus difference in results for large $m$ and small $m$ is less apparent.

	We examine the privacy-utility trade-offs by investigating performances under different $\varepsilon$s. 
	In \ref{fig:epsilonl2cor}, the error decreases as $\varepsilon$ increases for all private methods as expected.
	Moreover, the error of 2-SLR is comparable to Lasso, while error of M-SLR quickly drops below Lasso at medium privacy region $\varepsilon \geq 4$.
	This again ensures the superiority of our methods compared to fitting Lasso using only local information.

	\begin{figure}[htbp]
		\centering
		\subfigure[$n$ - $\ell_2$ error.]{
			\begin{minipage}{0.31\linewidth}
				\centering
				\includegraphics[width=\linewidth]{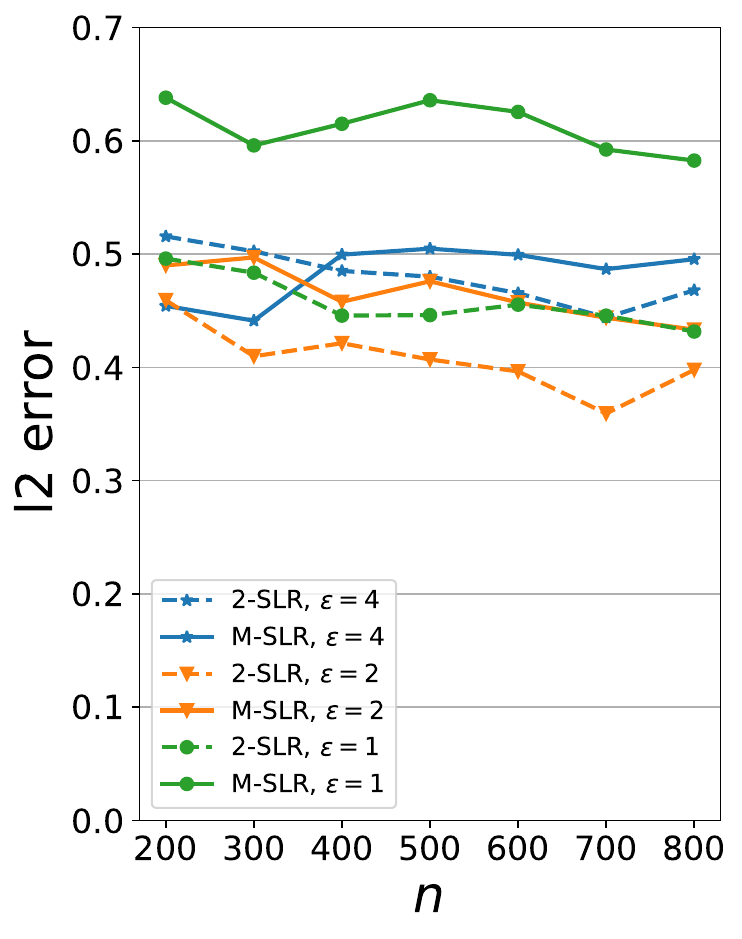}
			\end{minipage}
			\label{fig:nl2cor}
		}
		\hskip -0.1in
		\subfigure[$m$ - $\ell_2$ error.]{
			\begin{minipage}{0.31\linewidth}
				\centering
				\includegraphics[width=\linewidth]{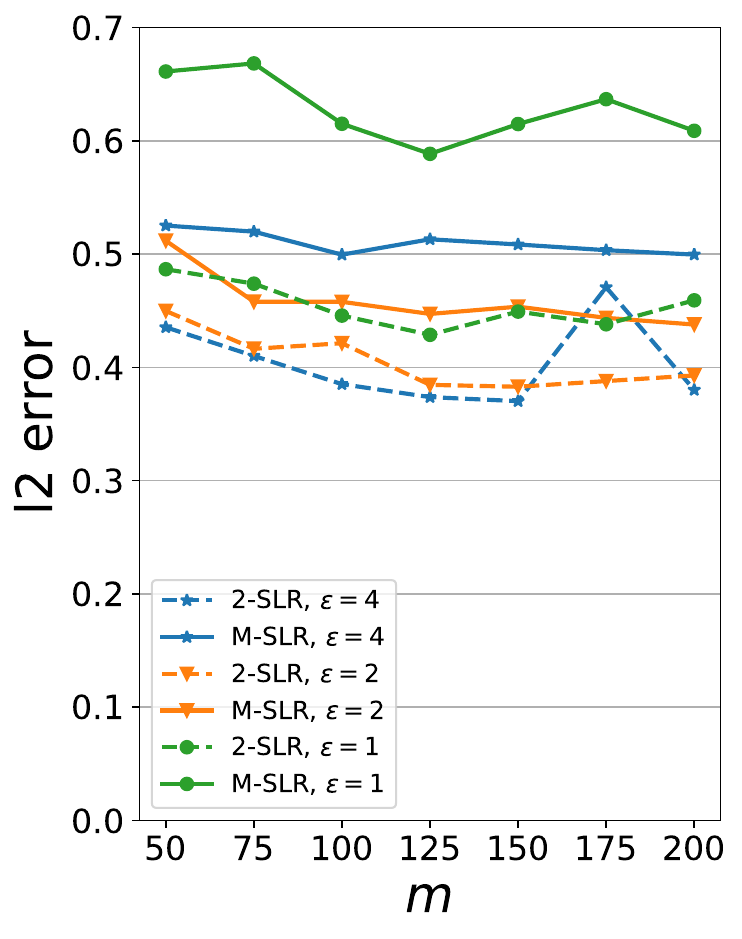}
			\end{minipage}
			\label{fig:ml2cor}
		}
		\hskip -0.1in
		\subfigure[$n / m$ - $\ell_2$ error.]{
			\begin{minipage}{0.31\linewidth}
				\centering
				\includegraphics[width=\linewidth]{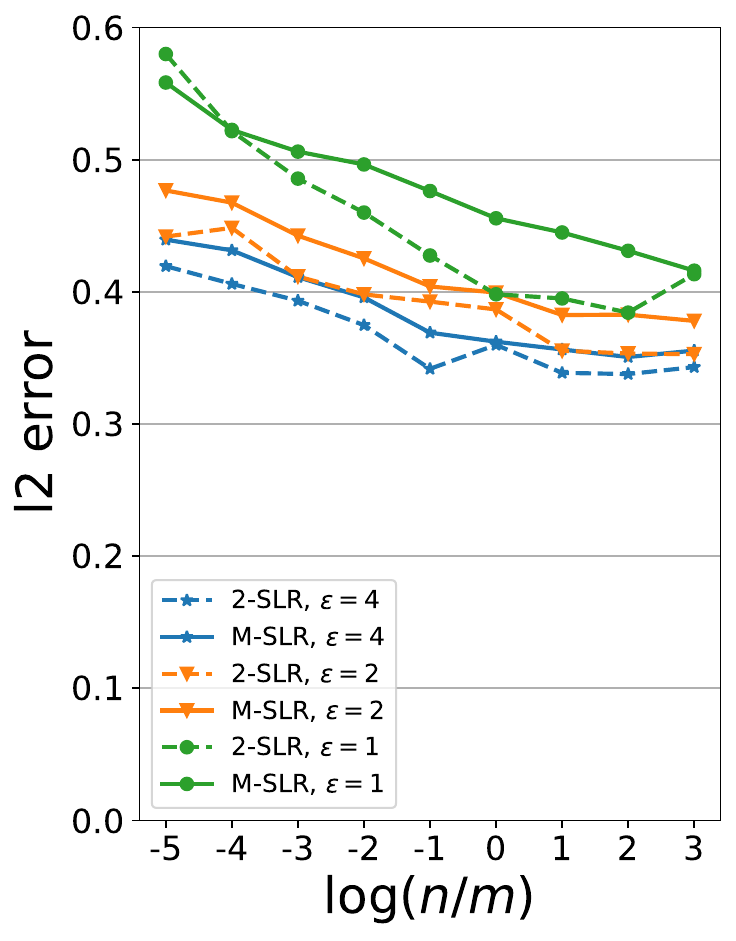}
			\end{minipage}
			\label{fig:nml2cor}
		}
		\vskip -0.15in
		\caption{Experiments w.r.t. sample sizes for correlated marginal.
		}
		\label{fig:nmcor}
	\end{figure}

	Finally, we analyze the impact of sample sizes. 
	In Figure \ref{fig:nl2cor} and \ref{fig:ml2cor}, the $\ell_2$ error decreases as both $n$ and $m$ increases for all $\varepsilon$, which confirms our theoretical claims. 
	The error is generally higher than that in the independent case.
	The overall $\ell_2$ curve is less sensitive to $n$ and $m$.
	Moreover, we let $nm = 400\times 100$ and vary the ratio $n / m$. 
	In \ref{fig:nml2}, we observe that, for each $\varepsilon$, the error of 2-SLR retains for $n/ m \approx 1$, while increase slightly when either $n$ or $m$ is too small, which is compatible with Theorem \ref{thm:preciseestimation}. 
	The performance of M-SLR is still sensitive to $n$ becoming small.

	\subsection{Real Dataset Description} \label{app:realdatasets}

	A summary of key information for these datasets after pre-processing can be found in Table \ref{tab:informationrealdatasets}. 
	For user-specific sample partitioning, certain datasets come with predefined partitions, while others undergo random partitioning. 
	Categorical features in the datasets are transformed into dummy variables, while each continuous feature is individually scaled to zero mean and unit variance. 
	We also present additional information of the data sets including the data source and the pre-processing details.

	\begin{table}[h]
		\centering
		\caption{Information of real datasets.}
		
		\label{tab:informationrealdatasets}
		\resizebox{0.6\linewidth}{!}{
			\renewcommand{\arraystretch}{1}
			\setlength{\tabcolsep}{5pt}
			\begin{tabular}{|l|r|r|r|r|r|}
				\toprule
				Dataset            & Sample Partition & \multicolumn{1}{c|}{d} & \multicolumn{1}{c|}{n} & \multicolumn{1}{c|}{m} & \multicolumn{1}{c|}{Area} \\ \midrule
				\texttt{Airline} & Predefined& 260   & 205 & 200-400     & Social    \\
				\texttt{Loan}   & Random & 735 & 500   & 100  & Business      \\
				\texttt{Mip}      & Predefined & 144   & 218 & 5     &  Computer Science   \\
				\texttt{Taxi}      & Predefined     & 213             & 1200       &  189-200   &  Social\\ 
				\texttt{Wine}    & Random    & 41 & 60 & 100     & Business     \\
				\texttt{Yolanda}    & Random & 100 & 800   & 200      & Social   \\
				\bottomrule
			\end{tabular}
		}
	\end{table}
	
	\texttt{Airline}: The Airlines-Departure-Delay dataset originally comes from United States Department of Transportation and currently available on OpenML \cite{AirlinesDepDelay}, consists of 1,048,575 observations, including one target variable and 9 attributes pertaining to flight information. We partition samples into users based on the "Destination" variable, selecting $205$ users with sample counts ranging from $200$ to $400$. Attributes such as "Origin" and "UniqueCarrier" are transformed into dummy variables, contributing to a total of $260$ features in the "Airlines" dataset. Overall, the Airlines dataset contains $75,600$ samples.
	
	\texttt{Loan}: The {Loan-Default-Prediction dataset is obtained from the training set of the Kaggle Loan Default Prediction challenge \citep{loan}, which aims to reduce the consumption of economic capital and optimize on the risk to the financial investor. The original dataset comprises 55319 instances of $735$ attributes
		We randomly select $50,000$ samples and partition the data into $500$ groups, with each group containing $100$ samples.
		
		\texttt{Mip}: The MIP-2016-regression dataset, available on OpenML, comprises $1,090$ instances featuring $144$ attributes and $1$ output attribute \citep{MIP}. Within this dataset, there are a total of 218 users, with each user possessing 5 samples.
		
		\texttt{Taxi}:
		The Taxi dataset is obtained from the Differential Privacy Temporal Map Challenge \citep{taxi}, which aims to develop algorithms that preserve data utility while guaranteeing individual privacy protection. The dataset contains quantitative and categorical information about taxi trips in Chicago, including time, distance, location, payment, and service provider. We partition the samples based on the unique identification number of taxis ($taxi_id$), resulting in 1200 taxis with sample counts ranging from $189$ to $200$. Other features include the time of each trip ($seconds$), the distance of each trip ($miles$), the time period during which each trip occurs($shift$), index of the zone where the trip starts ($pca$), index of the zone where the trip ends ($dca$), service provider ($company$), the method used to pay for the trip ($payment\_type$) and amount of tips ($tips$) and fares ($fare$). We use the other variable to predict the fares of the fares ($fare$) of the trips. Attributes such as $shift$, $pca$,$dca$, $company$ and $payment\_type$ are transformed into dummy variables, resulting in a total of $213$ features in the Taxi dataset.
		
		\texttt{Wine}: This dataset originates from the Wine Quality dataset \cite{cortez2009modeling} on UCI Machine Learning Repository, which combines data from both the "red wine" and "white wine" datasets. The original dataset comprises $11$ features associated with wine to predict the corresponding wine quality. In an effort to enhance dimensionality, Gaussian random noise in $30$ dimensions has been incorporated. $6000$ instances are collected in the dataset. The samples are randomly partitioned among $60$ users, with each user having $100$ samples.
		
		\texttt{Yolanda}: The Yolanda dataset \citep{guyon2019analysis} contains 400000 instances of $100$ attributes and $1$ output attribute. We randomly select $160,000$ samples and distribute them into $800$ groups, with each group containing $200$ samples.

		\subsection{Additional Real Datasets Results}

		\begin{table}[htbp]
			\caption{Running time(seconds) on real datasets.}
			\label{tab:runningtime}
			\centering
			\resizebox{0.5\linewidth}{!}{
				\renewcommand{\arraystretch}{1}
				\setlength{\tabcolsep}{5pt}
				\begin{tabular}{l|l|ll|ll}
					\toprule
					Datasets  & Lasso  & 2-SLR & M-SLR & LDPPROX & LDPIHT \\ \midrule
					\texttt{Airline}  & 0.7  & 15.8 & 15.3 & 766.5 &  6.1\\
					\texttt{Loan}  & 43.1  & 74.7 & 106.9& 4124.0 & 30.6
					\\
					\texttt{MIP}   & 0.1  & 0.1 & 2.3 & 5.5 & 4.3
					\\
					\texttt{Taxi} & 0.1  & 0.7 & 11.5 & 1569.1& 7.0
					\\
					\texttt{Wine}  & 0.2 & 0.7 & 2.9  & 3.5 & 2.6
					\\
					\texttt{Yolanda}&0.5 & 0.4 &12.5 &365.8 &4.4
					\\ 
					\bottomrule
				\end{tabular}
			}
		\end{table}

\end{document}